\documentclass{article}
\usepackage[utf8]{inputenc}
\pdfoutput=1
\usepackage[sc,osf]{mathpazo}
\usepackage[margin=1.1in,letterpaper]{geometry}
\usepackage[sort&compress,numbers]{natbib}
\usepackage[colorlinks=true,citecolor=blue,breaklinks]{hyperref}
\usepackage[hyphenbreaks]{breakurl} 
\usepackage{caption}
\usepackage{subcaption}

\makeatletter
\newlength\aftertitskip     \newlength\beforetitskip
\newlength\interauthorskip  \newlength\aftermaketitskip

\def\maketitle{\par
 \begingroup
   \def\thefootnote{\fnsymbol{footnote}}
   \def\@makefnmark{\hbox to 4pt{$^{\@thefnmark}$\hss}}
   \@maketitle \@thanks
 \endgroup
\setcounter{footnote}{0}
 \let\maketitle\relax \let\@maketitle\relax
 \gdef\@thanks{}\gdef\@author{}\gdef\@title{}\let\thanks\relax}

\def\@startauthor{\noindent \normalsize\bf}
\def\@endauthor{}
\def\@starteditor{\noindent \small {\bf Editor:~}}
\def\@endeditor{\normalsize}
\def\@maketitle{\vbox{\hsize\textwidth
 \linewidth\hsize \vskip \beforetitskip
 {\begin{center} \LARGE\@title \par \end{center}} \vskip \aftertitskip
 {\def\and{\unskip\enspace{\rm and}\enspace}%
  \def\addr{\small\it}%
  \def\email{\hfill\small\tt}%
  \def\name{\normalsize\bf}%
  \def\AND{\@endauthor\rm\hss \vskip \interauthorskip \@startauthor}
  \@startauthor \@author \@endauthor}
}}

\makeatother

\pdfoutput=1                    

\usepackage[utf8]{inputenc} 
\usepackage[T1]{fontenc}    
\usepackage{hyperref}       

\usepackage[usenames,dvipsnames]{xcolor}
\definecolor{darkblue}{rgb}{0,0,0.90}
\hypersetup{
  colorlinks = true,
  citecolor  = darkblue,
  linkcolor  = darkblue,
  citecolor  = darkblue,
  filecolor  = darkblue,
  urlcolor   = darkblue,
}

\usepackage{url}            
\usepackage{booktabs}       
\usepackage{amsfonts}       
\usepackage{nicefrac}       
\usepackage{microtype}      
\usepackage[utf8]{inputenc}

\usepackage{graphicx}
\usepackage{framed}
\usepackage{amssymb}
\usepackage{amsfonts}
\usepackage{mathrsfs}
\usepackage{mathtools}
\usepackage{array}
\usepackage{amsthm}
\usepackage{verbatim} 
\usepackage{enumerate}
\usepackage{bbm}
\usepackage{mathtools}
\usepackage{commath}
\usepackage{wrapfig}
\usepackage{amsbsy}
\usepackage{subcaption}
\usepackage{float}

\usepackage{algorithm}
\usepackage[noend]{algpseudocode}

\newtheorem{prop}{Proposition}
\newtheorem{lemma}[prop]{Lemma}
\newtheorem{Conjecture}[prop]{Conjecture}
\newtheorem{theorem}[prop]{Theorem}

\newtheorem{cor}[prop]{Corollary}

\theoremstyle{definition}

\newtheorem{defn}[prop]{Definition}

\newcommand{\nlsum}{\sum\nolimits}
\newcommand{\nlprod}{\prod\nolimits}

\title{Flexible Modeling of Diversity with Strongly Log-Concave Distributions}

\author{\name Joshua Robinson \email{joshrob@mit.edu}\\
  \name Suvrit Sra \email{suvrit@mit.edu}\\
  \name Stefanie Jegelka \email{stefje@csail.mit.edu}\\
  \addr{Massachusetts Institute of Technology, Cambridge, MA 02139} 
}

\begin{document}

\maketitle

\begin{abstract}

Strongly log-concave (SLC) distributions are a rich class of discrete probability distributions over subsets of some ground set. They are strictly more general than strongly Rayleigh (SR) distributions such as the  well-known determinantal point process. While SR distributions offer elegant models of diversity, they lack an easy control over how they express diversity. We propose SLC as the right extension of SR that enables easier, more intuitive control over diversity, illustrating this via examples of practical importance. We develop two fundamental tools needed to apply SLC distributions to learning and inference: \emph{sampling} and \emph{mode finding}.  For sampling we develop an MCMC sampler and give theoretical mixing time bounds. For mode finding, we establish a weak log-submodularity property for SLC functions and derive optimization guarantees for a distorted greedy algorithm. 
\end{abstract}

\vspace*{-12pt}
\section{Introduction}
\vspace*{-5pt}
A variety of machine learning tasks involve selecting diverse subsets of items. How we model diversity is, therefore, a key concern with possibly far-reaching consequences. Recently popular probabilisitic models of diversity include determinantal point processes~\citep{hoKrPe06,kulesza2012determinantal}, and more generally, strongly Rayleigh (SR) distributions~\citep{borcea2009negative,nipstut18}. These models have been successfully deployed for subset selection in applications such as video summarization~\citep{lin2012learning}, fairness~\citep{celis2018fair}, model compression~\citep{mariet2015diversity}, anomaly detection~\citep{mariet2018exponentiated}, the Nystr\"om method~\citep{li2016fast2}, generative models~\citep{elfeki2018gdpp, kwok2012priors}, and accelerated coordinate descent \cite{rodomanov2019randomized}. While valuable and broadly applicable, SR distributions have one main drawback: it is difficult to control the strength and nature of diversity they model. 

We counter this drawback by leveraging strongly log-concave (SLC) distributions~\citep{anari2018log,anari2018log3,anari2018log2}. These distributions are strictly more general than SR measures, and possess key properties that enable easier, more intuitive control over diversity. They derive their name from SLC polynomials introduced by Gurvits already a decade ago~\citep{gurvits2009multivariate}. More recently they have shot into prominence due to their key role in developing deep connections between discrete and continuous conxevity, with subsequent applications in combinatorics~\citep{adiprasito2018hodge, branden2019lorentzian, huh2017combinatorial}. In particular, they lie at the heart of recent breakthrough results such as a proof of Mason's conjecture~\citep{anari2018log3} and obtaining a fully polynomial-time approximation scheme for counting the number of bases of arbitrary matroids~\citep{anari2018log, anari2018log2}. We remark that all these works assume \emph{homogeneous} SLC polynomials.

We build on this progress to develop fundamental tools for general SLC distributions, namely,  \emph{sampling} and \emph{mode finding}. 
We highlight the flexibility of SLC distributions through two settings of importance in practice: (i) raising any SLC distribution to a power $\alpha \in [0,1]$; and (ii) incorporating a constraint that allows sampling sets of \emph{any} size up to a budget. In contrast to similar modifications to SR measures (see e.g.,~\citep{mariet2018exponentiated}), these settings retain the crucial SLC property. Setting~(i) allows us to conveniently tune the strength of diversity by varying a single parameter; while setting~(ii) offers greater flexibility than \emph{fixed} cardinality distributions such as a $k$-determinantal point process~\citep{kulesza2011k}. This observation is simple yet important, especially since the ``right'' value of $k$ is hard to fix \emph{a priori}. 
\paragraph{Contributions.}
We briefly summarize the main contributions of this work below.
\vspace{-8pt}
\begin{list}{{\tiny$\blacksquare$}}{\leftmargin=1.5em}
  \setlength{\itemsep}{1pt}
\item We introduce the class of strongly log-concave distributions to the machine learning community, showing how it can offer a flexible discrete probabilistic model for distributions over subsets.
\item We prove various closure properties of SLC distributions (Theorems \ref{thm,: weighted homogenization}-\ref{corollary: k swh closure theorem}), and show how to use these properties for better controlling the distributions used for inference.
\item We derive sampling algorithms for SLC and related distributions, and analyze their corresponding mixing times both theoretically and empirically (Algorithm \ref{alg: rescaled general SLC sampling MCMC}, Theorem \ref{thm: MCMC mixing time bound rescaled}). 
\item We study the negative dependence of SLC distributions by deriving a weak log-submodularity property (Theorem \ref{thm: weak submodular property}). Optimization guarantees for a selection of greedy algorithms are obtained as a consequence (Theorem \ref{thm: distorted greedy}). 
\end{list}
\vspace{-8pt}
As noted above, our results build on the remarkable recent progress in~\citep{anari2018log,anari2018log3,anari2018log2} and \citep{branden2019lorentzian}. The biggest difference between the previous work and this work is our focus on general non-homogeneous SLC polynomials, corresponding to distributions over sets of varying cardinality, as opposed to purely the homogeneous, i.e., fixed-cardinality, case. This broader focus necessitates development of some new machinery, because unlike SR polynomials, the class of SLC polynomials is not closed under homogenization. We summarize the related work below for additional context. 
\vspace*{-5pt}
\subsection{Related work}
\textbf{SR polynomials.} Strongly Rayleigh distributions were introduced in \cite{borcea2009negative} as a class of discrete distributions possessing several strong negative dependence properties. It did not take long for their potential in machine learning to be identified \cite{kulesza2012determinantal}. Particular attention has been paid to determinantal point processes due to the intuitive way they capture negative dependence, and the fact that they are parameterized by a single positive semi-definite kernel matrix. Convenient parameterization has allowed an abundance of fast algorithms for learning the kernel matrix \cite{dupuy2016learning, gartrell2017low, mariet2015fixed, mariet2016kronecker}, and sampling \cite{anari2016monte, li2016fast, mariet2019dppnet}. SR distributions are a fascinating and elegant probabilistic family whose applicability in machine learning is still an emerging topic \cite{derezinski2017unbiased,nipstut18,li2017polynomial,mariet17symmetric}.

\textbf{SLC polynomials.} Gurvits introduced SLC polynomials a decade ago \cite{gurvits2009multivariate} and studied their connection to discrete convex geometry. Recently this connection was significantly developed \citep{branden2019lorentzian,anari2018log2} by establishing that matroids, and more generally M-convex sets, are characterized by the strong log-concavity of their generating polynomial. This is in contrast to SR, for which it is known that some matroids have generating polynomials that are \emph{not} SR \citep{branden2007polynomials}.

\textbf{Log-Submodular Distributions.} Distributions over subsets that are log-submodular (or supermodular) are amenable to mode finding and variational inference with approximation guarantees, by exploiting the optimization properties of submodular functions \cite{djolonga14variational,djolonga15scalable,djolonga18provable}. Theoretical bounds on sampling time require additional assumptions \cite{gotovos15sampling}. \citet{iyer15spp} analyze inference for submodular distributions, establishing polynomial approximation bounds.

\textbf{MCMC samplers and mixing time.}
The seminal works \cite{diaconis1991geometric, diaconis1996logarithmic} offer two tools for obtaining mixing time bounds for Markov chains: lower bounding the spectral gap, or log-Sobolev constant. These techniques have been successfully deployed to obtain mixing time bounds for homogenous SR distributions \cite{anari2016monte}, general SR distributions \cite{li2016fast}, and recently homogenous SLC distributions \cite{anari2018log2}.

 \vspace*{-5pt} 
\section{Background and setup}
\paragraph{Notation.}
We write $[n] = \{ 1, \ldots , n \}$, and denote by $2^{[n]}$ the power set $\{ S \mid S \subseteq [n] \}$. For any variable $u$, write $\partial_u$ to denote $\frac{\partial}{\partial u}$; in case $u = z_i $, we often abbreviate further by writing $\partial_i$ instead of $\partial_{z_i}$. For $S \subseteq [n]$ and $\alpha \in \mathbb{N}^n$ let $\mathbf{1}_S \in \{0,1\}^n$ denote the binary indicator vector of $S$, and define $\abs{\alpha} = \sum_{i=1}^n \alpha_i$. We also write variously $\partial_z^S =  \prod_{i \in S} \partial_i$ and $\partial_z^\alpha = \prod_{i \in [n]} \partial^{\alpha_i}_i$ where $\alpha_i =0$ means we do not take any derivatives with respect to $z_i$. We let $z^S$ and $z^\alpha$ denote the monomials $\prod_{i \in S} z_i$ and $\prod_{i =1}^n z_i^{\alpha_i}$ respectively. 
For $K = \mathbb{R}$ or $\mathbb{R}_+$ we write $K[z_1, \ldots , z_n]$ to denote the set of all polynomials in the variables $z = (z_1, \ldots , z_n)$ whose coefficients belong to $K$. A polynomial is said to be $d$-homogeneous if it is the sum of monomials \emph{all} of which are of degree $d$. 
Finally, for a set $X$ we shall minimize clutter by using $X \cup i$ and $X \setminus i$ to denote $X \cup \{i\}$ and $X \setminus \{i \}$ respectively. 

\paragraph{SLC distributions.}
We consider distributions $\pi : 2^{[n]} \rightarrow [0,1]$ on the subsets of a ground set $[n]$. There is a one-to-one correspondence between such distributions, and their generating polynomials
\begin{equation}
  \label{eq:1}
  f_\pi(z) := \nlsum_{S \subseteq [n] } \pi(S) \nlprod_{i \in S} z_i  = \nlsum_{S \subseteq [n] } \pi(S)  z^S.
\end{equation}
The central object of interest in this paper is the class of strongly log-concave distributions, which is  defined by imposing certain log-concavity requirements on the corresponding generating polynomials. 

\begin{defn}
  \label{def:slc}
  A polynomial $f \in \mathbb{R}_+[z_1, \ldots , z_n]$ is \emph{strongly log-concave} (SLC) if every derivative of $f$ is log-concave. That is, for any $\alpha \in \mathbb{N}^n$ either $\partial^\alpha f = 0$, or the function $\log ( \partial^\alpha f(z) ) $ is concave at all $z \in \mathbb{R}_+^n$. We say a distribution $\pi$ is strongly log-concave if its generating polynomial $f_\pi$ is strongly log-concave; we also say $\pi$ is $d$-homogeneous if $f_\pi$ is $d$-homogeneous. 
\end{defn}

There are many examples of SLC distributions; we note a few important ones below.
\vspace{-5pt}
\begin{list}{–}{\leftmargin=1.5em}
  \setlength{\itemsep}{0pt}
\item Determinantal point processes~\citep{kulesza2012determinantal, gillenwater2012near, kulesza2011k, li2016fast2}, and more generally, Strongly Rayleigh (SR) distributions~\citep{borcea2009negative, derezinski2017unbiased, li2017polynomial,nipstut18}.
\item Exponentiated (for exponents in $[0,1]$) homogeneous SR distributions~\citep{mariet2018exponentiated, anari2018log2}.
\item The uniform distribution on the independent sets of a matroid \cite{anari2018log3}.
\end{list}
\vspace{-5pt}
SR distributions satisfy several strong negative dependence properties (e.g., log-submodularity and negative association). The fact that SLC is a strict superset of SR suggests that SLC distributions possess some weaker negative dependence properties. These properties will play a crucial role in the two fundamental tasks that we study in this paper: \emph{sampling} and \emph{mode finding}.

\paragraph{Sampling.}
Our first task is to efficiently draw samples from an SLC distribution $\pi$. To that end, we seek to develop Markov Chain Monte Carlo (MCMC) samplers whose mixing time (see Section \ref{sec:mcmc} for definition) can be well-controlled. For homogeneous $\pi$, the breakthrough work of~\citet{anari2018log2} provides the first analysis of fast-mixing for a simple Markov chain called \texttt{Base Exchange Walk}; this analysis is further refined in~\citep{cryan2019modified}. { \tt Base Exchange Walk} is defined as follows: if currently at state $S \subseteq [n]$, remove an element $i \in S$ uniformly at random. Then move to $R \supset S \setminus \{i\} $ with probability proportional to $\pi(R)$. This describes a transition kernel $Q(S,R)$ for moving from $S$ to $R$.  We build on these works to obtain the first mixing time bounds for sampling from \emph{general} (i.e., not necessarily homogeneous) SLC distributions (Section~\ref{sec:mcmc}).

\paragraph{Mode finding.} 
Our second main goal is optimization, where we consider the more general task of finding a mode of an SLC distribution subject to a cardinality constraint. This task involves solving $\max_{\abs{S} \leq d} \pi(S)$. This task is known to be NP-hard even for SR distributions; indeed, the maximum volume subdeterminant problem~\citep{civril2013exponential} is a special case.We consider a more practical approach based on observing that SLC distributions satisfy a relaxed notion of log-submodularity, which enables us to adapt simple greedy algorithms. 
Before presenting the details about sampling and optimization, we need to first establish some key theoretical properties of general SLC distributions. This is the subject of the next section.
\vspace*{-5pt}
\section{Theoretical tools for general SLC polynomials}\label{section: preserving SLC}

In this technical section we develop the theory of strong log-concavity by detailing several transformations of an SLC polynomial $f$ that preserve strong log-concavity. Such closure properties can be essential for proving the SLC property, or for developing algorithmic results. Due to the correspondence between distributions on $2^{[n]}$ and their generating polynomials, each statement concerning polynomials can be translated into a statement about probability distributions. The following theorem is a crucial stepping stone to sampling from non-homogeneous SLC distributions, and to sampling with cardinality constraints.

\begin{theorem}\label{thm,: weighted homogenization}
Let  $f = \sum_{S \subseteq [n]}c_S z^S \in \mathbb{R}_+[z_1, \ldots , z_n]$ be SLC, and suppose the support of the sum is the collection of independent sets of a rank $d$ matroid. Then for any $k \leq d$ the following polynomial is SLC:

\[\mathcal{H}_{k} f(z,y) = \sum_{\abs{S} \leq k} \frac{c_S}{(k - \abs{S})!} z^Sy^{k-\abs{S}}. \]
\end{theorem}

The above operation is also referred to as \emph{scaled homogenization}, since the resulting polynomial is homogeneous and there is an added $1/(k-\abs{S})!$ factor. In fact, we may extend Theorem~\ref{thm,: weighted homogenization} to allowing the user to add an additional exponentiating factor:
\begin{theorem}\label{thm: weighted exponentiation}
Let $f = \sum_{S \subseteq [n]}c_S z^S \in \mathbb{R}_+[z_1, \ldots , z_n]$ be SLC, and suppose the support of the sum is the collection of independent sets of a rank $d$ matroid. Then for $0 \leq \alpha \leq 1$ and any $k \leq d$ the following polynomial is SLC:

\[\mathcal{H}_{k, \alpha} f(z,y) = \sum_{ \abs{S} \leq k } \frac{c_S^\alpha}{(k - \abs{S})!} z^Sy^{k-\abs{S}}. \]
\end{theorem}

Notably, Theorem \ref{thm: weighted exponentiation} fails for \emph{all} $\alpha > 1$. For a proof of this see Appendix \ref{sec: closure under exponentiation}.

Next, we show that polarization preserves strong log-concavity. Polarization essentially means to replace a variable with a higher power by multiple ``copies'', each occurring only with power one, in a way that the resulting polynomial is symmetric (or permutation-invariant) in those copies. This is achieved by averaging over elementary symmetric polynomials. Formally, 
the polarization of the polynomial $f = \sum_{\abs{S} \leq d } c_S z^S y^{d-\abs{S}} \in \mathbb{R}[z_1, \ldots , z_n ,y]$ is defined to be

 \vspace*{-5pt} 
 
\[ \Pi f(z_1, \ldots , z_n , y_1, \ldots , y_d ) = \sum_{\abs{S} \leq d}  c_S  z^S  {d \choose \abs{S}}^{-1}    e_{d - \abs{S}}(y_1, \ldots , y_d) \]

 \vspace*{-5pt} 
 
 where $e_k(y_1, \ldots , y_d)$ is the $k$th elementary symmetric polynomial in $d$ variables. The polarization $\Pi f$ has the following three properties:
 \vspace{-5pt}
\begin{enumerate}\setlength{\itemsep}{0pt}
\item It is symmetric in the variables $y_1, \ldots , y_d$;
\item Setting $y_1 = \ldots = y_d= y$  recovers $f$;
\item $\Pi f $ is multiaffine, and hence the generating polynomial of a distribution on $2^{[n+d]}$.
\end{enumerate}
 \vspace{-5pt}
Closure under polarization, combined with the homogenization results (Theorems \ref{thm,: weighted homogenization} and \ref{thm: weighted exponentiation}) allows non-homogeneous distributions to be transformed into homogenous ones. This allows general SLC distributions to be transformed into homogenous SLC distributions for which fast mixing results are known \cite{anari2018log2}. How to work backwards to obtain samples from the original distribution will be the topic of the next section.

\begin{theorem}\footnote{This result was independently discovered by Br{\"a}nd{\'e}n and Huh \cite{branden2019lorentzian}. }\label{thm: closure under polarization} 
Let $f = \sum_{ S \subseteq [n]} c_S  z^S y^{d-\abs{S}} \in \mathbb{R}_+[z_1, \ldots , z_n, y]$ be SLC, and the support of the sum is the collection of independent sets of a rank $d$ matroid. Then the polarization $\Pi f $ is SLC.
\end{theorem}

Putting all of the preceding results together we obtain the following important corollary. It is this observation that will allow us to do mode finding for SLC distributions and exponentiated, cardinality constrained SLC distributions. 

\begin{cor}\label{corollary: k swh closure theorem}
Let $f = \sum_{S \subseteq [n]}c_S z^S \in \mathbb{R}_+[z_1, \ldots , z_n]$ be SLC, and suppose the support of the sum is the collection of independent sets of a rank $d$ matroid. Then $\Pi(\mathcal{H}_{k, \alpha} f)$ is SLC for any $k \leq d$ and $0 \leq \alpha \leq 1$.
\end{cor}

 In Appendix \ref{section: Closure under constraining to subsets of a given size } we also show that SLC distributions are closed under conditioning on a fixed set size. We mention those results since they may be of independent interest, but omit them from the main text since we do not use them further in this paper. 

 \vspace*{-5pt} 
\section{Sampling from strongly log-concave distributions}
\label{sec:mcmc}

In this section we outline how to use the SLC closure results from Section~\ref{section: preserving SLC} to build a sampling algorithm for general SLC distributions and prove mixing time bounds. Recall that we are considering a probability distribution $\pi : 2^{[n]} \rightarrow [0,1]$ that is strongly log-concave. The mixing time of a Markov chain $(Q,\pi)$ started at $S_0$ is $ t_{S_0} (\varepsilon ) = \min \{ t \in \mathbb{N} \mid \norm{ Q^t(S_0, \cdot) - \pi}_1 \leq \varepsilon \}$ where $Q^t$ is the $t$-step transition kernel. For the remainder of this section we consider the distribution $\nu$ where $\nu(S) \propto \pi(S)^\alpha \mathbf{1}\{ \abs{S} \leq d\}$ for $0 \leq \alpha \leq 1$, and $d\in [n]$. In particular, this includes $\pi$ itself. The power $\alpha$ allows to vary the degree of diversity induced by the distribution.

Our strategy is as follows: we first ``extend'' $\nu$ to a distribution $\nu_{\text{sh}}$ over subsets of size $|n|$ of $[n+d]$ to obtain a homogeneous distribution. If we can sample from $\nu_{\text{sh}}$, then we can  extract a sample $S \subseteq [n]$ of a scaled version ov $\nu$ by simply restricting a sample $T \sim \nu_{\text{sh}}$ to $T \cap [n]$. If $\nu$ was SR, then $\nu_{\text{sh}}$ would also be SR, and a fast sampler follows from this observation \cite{li2016fast}. But, for general SLC distributions (and their powers), $\nu_{\text{sh}}$ is not SLC, and deriving a sampler is more challenging.

To still enable the homogenization strategy, we instead derive a carefully scaled version of a homogeneous version of $\nu$ that, as we prove, is homogeneneous and SLC and hence tractable. We use this rescaled version as a proposal distribution in a sampler for $\nu_{\text{sh}}$.

To obtain an appropriately scaled extended, homogeneous variant $\nu$, we first translate Corollary \ref{corollary: k swh closure theorem} into probabilistic language. 

\begin{theorem}
Suppose that the support of the sum in the generating polynomial of $\nu$ is the collection of independent sets of a rank $d$ matroid. Then  for any $k \leq d$ the following probability distribution on $2^{[n+k]}$ is SLC:

\[\mathcal{H}_{k} \nu(S) \propto \begin{cases}
     \binom{k}{\abs{S\cap [n]} }^{-1}\frac{\nu(S \cap [n])}{(k - \abs{S \cap [n]})!},& \text{ for all $S \subseteq [n+k]$ such that $\abs{S} = k$}\\
    0,              & \text{otherwise.}
\end{cases}
\]
\end{theorem}

\begin{proof}
Observe that the generating polynomial of $\mathcal{H}_{k} \nu$ is $ \Pi (\mathcal{H}_{k}f)$ where $f$ denotes the generating polynomial of $\nu$. The result follows immediately from Corollary \ref{corollary: k swh closure theorem}.  
\end{proof}

 \vspace*{-5pt} 
The ultimate proposal that we use is not $\mathcal{H}_{k} \nu$, but a modified version $\mu$ that better aligns with $\nu$:

  \vspace*{-5pt} 
 \begin{equation*}
\mu(S) \propto  \bigg (\frac{d}{e} \bigg )^{ d - \abs{S \cap [n]}} \mathcal{H}_{d} \nu(S).
\end{equation*}
 \vspace*{-5pt}
 \begin{prop}
   If $\nu$ is SLC, then $\mu$ is SLC.
 \end{prop}
  \vspace*{-5pt} 
 \begin{proof}
 Lemma \ref{lemma: SLC closed under affine transformations} in the Appendix says that strong log-concavity is preserved under linear transformations of the coordinates. This implies that $\mu$ is SLC since its generating polynomial is $ \Pi ( (\mathcal{H}_{d}f) \circ T)$  where $f$ is the generating polynomial of $\nu$ and $T$ is the linear transform defined by: $y \mapsto \frac{d}{e} y$ and $z_i \mapsto z_i$ for $i = 1 \ldots , n$.
\end{proof}
 \vspace*{-5pt} 
 Importantly, since $\mu$ is homogeneous and SLC, the { \tt Base Exchange Walk } for $\mu$ mixes rapidly. Let $Q$ denote the Markov transition kernel for { \tt Base Exchange Walk} on $2^{[n+d]}$ for $\mu$. We use $Q$ as a proposal, and then compute the appropriate acceptance probability to obtain a chain that mixes to the symmetric homogenization $\nu_{ \text{sh} }$ of $\nu$. The target $\nu_{ \text{sh} }$ is a $d$-homogenous distribution on $2^{[n+d]}$:

 \vspace*{-5pt} 
\[\nu_{ \text{sh} }(S) \propto 
     \binom{d}{\abs{S\cap [n]} }^{-1}\nu(S \cap [n]) , \text{ for all $S \subseteq [n+d]$ such that $\abs{S} = d$.}
\]
 \vspace*{-5pt} 
 
A crucial property of $\nu_{ \text{sh} }$ is that its marginalization over the ``dummy'' variables yields $\nu$, i.e., $\sum_{T: T \cap [n] = S} \nu_{ \text{sh} }(T) = \nu(S)$. Therefore, after obtaining a sample $T \sim  \nu_{ \text{sh} }$ one then obtains a sample from $\nu$ by computing $T \cap [n]$. 

 \vspace*{-5pt} 
\begin{algorithm}
\caption{Metropolis-Hastings sampler for $\nu_{ \text{sh} }$ with proposal $Q$ }\label{alg: rescaled general SLC sampling MCMC}
\begin{algorithmic}[1]
\State Initialize $S \subseteq [n+d]$
\While{not mixed}
\State Set $k \gets \abs{S \cap [n]}$ 
\State Propose move $T \sim Q(S, \cdot)$
\If { $\abs{T \cap [n]} = k-1$}
\State $ R \gets T$ with probability $\min \{ 1, \frac{e}{d}  (d-k+1)  \}$, otherwise stay at $S$
\EndIf
\If { $\abs{T \cap [n]} = k$}
\State $R \gets T$
\EndIf
\If { $\abs{T \cap [n]} = k+1$}
\State $ R \gets T$ with probability $\min \{ 1, \frac{d}{e} \frac{1}{(d-k)} \}$, otherwise stay at $S$
\EndIf
\EndWhile
\end{algorithmic}
\end{algorithm}

It is a simple computation to show that the acceptance probabilities in Algorithm \ref{alg: rescaled general SLC sampling MCMC} are indeed the Metropolis-Hastings acceptance probabilities for sampling from $\nu_{ \text{sh} }$ using the proposal $Q$. Therefore the chain mixes to $\nu_{ \text{sh} }$. We obtain the following mixing time bound, recalling that the mixing time of $(Q, \nu_{ \text{sh} })$ is $ t_{S_0} (\varepsilon ) = \min \{ t \in \mathbb{N} \mid \norm{ Q^t(S_0, \cdot) - \nu_{ \text{sh}}}_1 \leq \varepsilon \}$.

\begin{theorem}\label{thm: MCMC mixing time bound rescaled}
For $d \geq 8$ the mixing time of the chain in Algorithm \ref{alg: rescaled general SLC sampling MCMC} started at $S_0$ satisfies the bound

 \vspace*{-5pt} 
\[ t_{S_0} (\varepsilon ) \leq \frac{1}{e \sqrt{2\pi} }d ^{5/2} 2^d \bigg ( \log \log \bigg \{ { d \choose \abs{S_0} } \frac{1}{\nu(S_0)}  \bigg \}+ \log \frac{1}{2 \varepsilon^2} \bigg ). \]
 \vspace*{-5pt} 
 
\end{theorem}

 \vspace*{-5pt} 
\section{Maximization of weakly log-submodular functions}\label{section: weak submodular}

In this section we explore the negative dependence properties of SLC functions (unnormalized SLC distributions). To do this we introduce a new notion of weak submodularity. Then we show that any function $\nu$ such that  $\mathcal{H}_{d} \nu$ is SLC is weak log-submodular. In particular, this includes all examples discussed above. Finally, we prove that a distorted greedy optimization procedure leads to optimization guarantees for weak (log-)submodular functions for the cardinality constrained problem $\text{OPT} \in \arg \max_{ \abs{S} \leq k} \nu(S)$. Appendix \ref{appendix: submodular} contains similar results for constrained greedy optimization of \emph{increasing}  weak (log-)submodular functions and unconstrained double greedy optimization of non-negative (log-)submodular functions.

\begin{defn}
We call a function $\rho : 2^{[n]} \rightarrow \mathbb{R}$ $\gamma$-weakly submodular if for any $S \subseteq [n]$ and $i,j \in [n] \setminus S $ with $i$ and $j$ not equal, we have 

 \vspace*{-5pt} 
\[ \rho(S) + \rho(S \cup \{i , j \} ) \leq  \gamma + \rho(S \cup i ) + \rho(S \cup  j ). \]
 \vspace*{-5pt} 
 
We say $\nu : 2^{[n]} \rightarrow \mathbb{R}_+$ is $\gamma$-weakly log-submodular if $\log \nu$ is $(\log \gamma)$-weakly submodular.
\end{defn}
 \vspace*{-5pt} 
 
Note carefully that our notion of weak submodularity differs from a notion of weak submodularity that already appears in the literature \cite{das2011submodular, harshaw2019submodular, khanna2017scalable}. Building on a result by Brändén and Huh \cite{branden2019lorentzian}, we prove the following result. 

 \vspace*{-5pt} 
\begin{theorem}\label{thm: weak submodular property}
Any non-negative function $\rho : 2^{[n]} \rightarrow \mathbb{R}_+$ with support contained in $\{ S \subseteq [n] : \abs{S} \leq d \}$ and generating polynomial $f$ such that $\mathcal{H}_{d} f$ is strongly log-concave is $\gamma$-weakly log-submodular for $ \gamma = 4 \big ( 1 - \frac{1}{d} \big )$. 
\end{theorem}
 \vspace*{-5pt} 
 
This result, whilst weaker than log-submodularity, gives a path to optimizing strongly log-concave functions. Consider $\rho :  2^{[n]} \rightarrow \mathbb{R}$, assumed to be $\gamma$-weakly submodular. Note in particular we do \emph{not} assume that $\rho$ is non-negative. This is important since we are interested in applying this procedure to the logarithm of a distribution, which need not be non-negative. Define $c_e = \max\{\rho([n] \setminus e ) - \rho([n]),0\}$, and $c(S) = \sum_{e \in S} c_e$. We use the convention that $c(\varnothing) = 0$. Then we may decompose $\rho = \eta -c$ where $\eta = \rho +c$. Note that $\eta$ is $\gamma$-weakly submodular and $c$ is a non-negative function.

We will extend the distorted greedy algorithm by \cite{feldman2018guess, harshaw2019submodular} to our notion of weak submodularity. To do so, we introduce the distorted objective $\Phi_i(S) = (1 - 1/k)^{k-i} \eta(S) - c(S)$ for $i= 0 , \ldots k$. The distorted greedy algorithm greedily builds a set $R$ of size at most $d$ by forming a sequence $\varnothing = S_0, S_1, \ldots , S_{k-1}, S_k=R$ such that $S_{i+1}$ is formed by adding the element $e_i \in [n]$ to $S_{i}$ that maximizes $\Phi_{i+1}(S_{i} \cup e_i) - \Phi_{i+1}(S_{i})$ so long as the increment is positive.

\vspace*{-5pt} 
\begin{algorithm}
\caption{Distorted greedy weak submodular constrained maximization of $\nu = \eta - c$}\label{alg: distorted greedy weak submodular maximization}
\begin{algorithmic}[1]
\State Let $S_0 = \varnothing$
\For{ $i = 0, \ldots , k-1$}
\State Set $e_i= \arg \max_{e \in [n]} \Phi_{i+1}(S_{i} \cup e_i) - \Phi_{i+1}(S_{i})$
\If {$\Phi_{i+1}(S_{i} \cup e_i) - \Phi_{i+1}(S_{i}) > 0$}
\State $S_{i+1} \gets S_{i} \cup e_i$
\Else \hspace{0cm} {$ S_{i+1} \gets S_{i} $}
\EndIf
\EndFor
\State \textbf{return} $R = S_k$
\end{algorithmic}
\end{algorithm}
 \vspace*{-5pt} 

\begin{theorem}\label{thm: distorted greedy}
Suppose $\rho : 2^{[n]} \rightarrow \mathbb{R}$ is $\gamma$-weakly submodular and $\rho(\varnothing) = 0$. Then the solution $R = S_k$ obtained by the distorted greedy algorithm satisfies
\[ \rho(R) = \eta(R) - c(R)  \geq \bigg ( 1 - \frac{1}{e} \bigg ) \bigg ( \eta(\text{OPT} ) - \frac{1}{2} \ell ( \ell -1) \gamma \bigg ) - c(\text{OPT}), \]
where $\ell := \abs{\text{OPT}} \leq k$.
\end{theorem}
 \vspace*{-5pt} 
 
Note any weakly submodular function can be brought into the required form by subtracting $\rho(\varnothing)$ if it is non-zero. If $\nu$ is weakly log-submodular, we can decompose $\nu = \eta / c$ such that $\log \eta$ and $\log c$ perform the same role as $\eta$ and $c$ did in the weakly submodular setting. Then by appling Theorem \ref{thm: distorted greedy} to $\log \nu$ we obtain the following corollary.

\begin{cor}\label{corollary: weak log-submodular distorted greedy}
Suppose $\nu : 2^{[n]} \rightarrow \mathbb{R}_+$ is $\gamma$-weakly log-submodular and $\nu(\varnothing) = 1$. Then the solution $R = S_k$ obtained by the distorted greedy algorithm satisfies
\begin{align*}
     \nu(R) = \frac{\eta(R)}{ c(R)}  \; \geq \; \gamma^{ - \frac{1}{2} \ell( \ell - 1) (1 - 1/e)} \frac{ \eta(\text{OPT})^{1- 1/e} }{c(\text{OPT})}.
\end{align*}
\end{cor}

 \vspace*{-5pt} 
\section{Experiments}\label{section: experiments}

In this section we empirically evaluate the mixing time of Algorithm \ref{alg: rescaled general SLC sampling MCMC}. We use the standard \emph{potential scale reduction factor} metric to measure convergence to the stationary distribution \cite{brooks1998general}. The method involves running several chains in parallel and computing the average variance within each chain and between the chains. The PSRF score is the ratio of the between variance over the within variance and is usually above $1$.  When the PSRF score is close to $1$ then the chains are considered to be mixed. In all of our experiments we run three chains in parallel and declare them to be mixed once the PSRF score drops below $1.05$. 

\begin{figure}
\centering
\begin{subfigure}{.5\textwidth}
  \centering
  \includegraphics[width=0.9\linewidth]{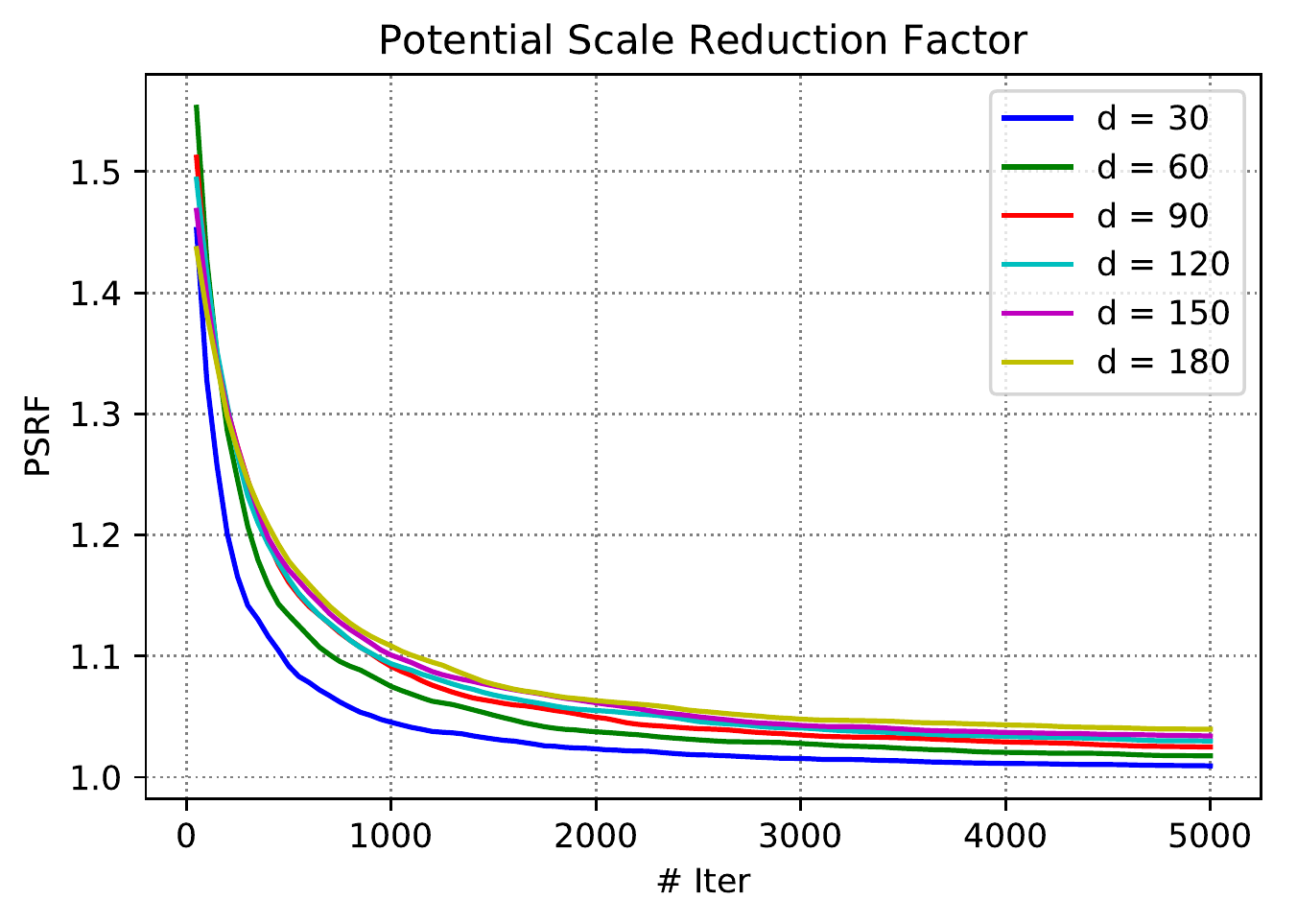}
  \caption{}
\end{subfigure}%
\begin{subfigure}{.5\textwidth}
  \centering
  \includegraphics[width=0.9\linewidth]{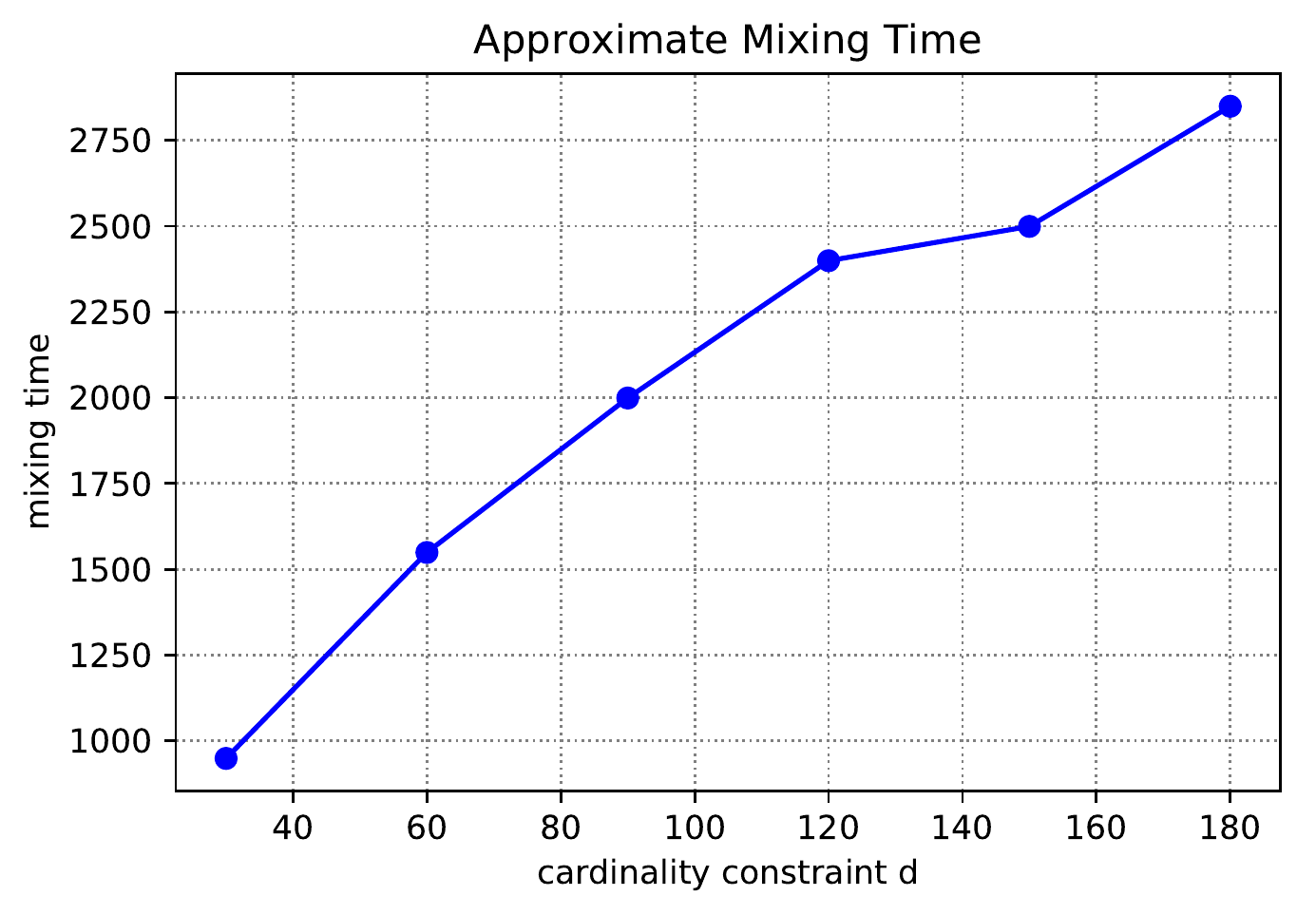}
  \caption{}
\end{subfigure}
\caption{Empirical mixing time analysis for sampling a ground set of size $n = 250$ and various cardinality constraints $d$, (a) the PSRF score for each set of chains, (b) the approximate mixing time obtained by thresholding at PSRF equal to $1.05$.}
\label{fig: PSRF and mix time varying d}
\end{figure}

\begin{figure}
\centering
\begin{subfigure}{.33\textwidth}
  \centering
  \includegraphics[width=1.0\linewidth]{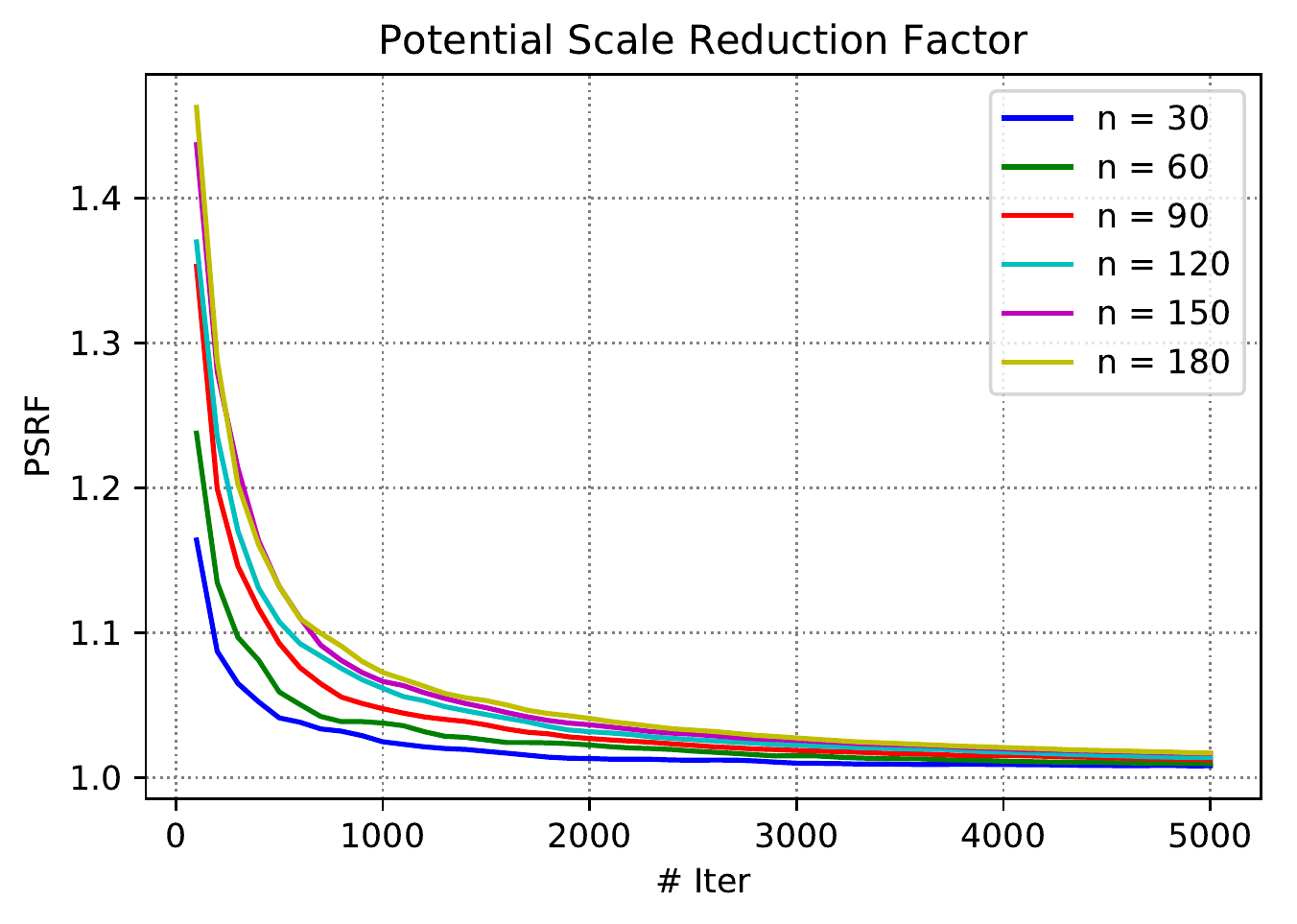}
  \caption{}
\end{subfigure}%
\begin{subfigure}{.33\textwidth}
  \centering
  \includegraphics[width=1.0\linewidth]{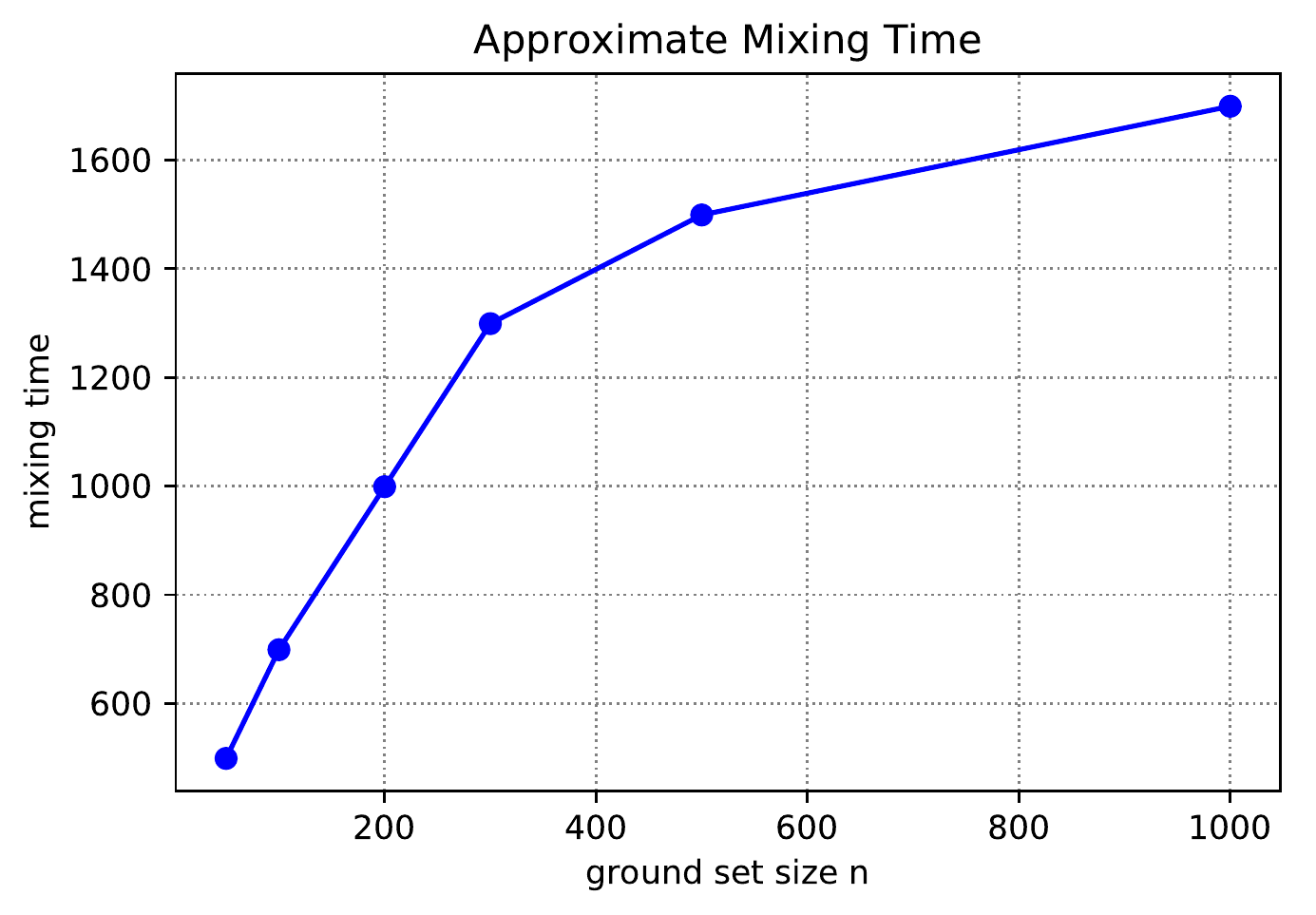}
  \caption{}
\end{subfigure}
\begin{subfigure}{.33\textwidth}
  \centering
  \includegraphics[width=1.0\linewidth]{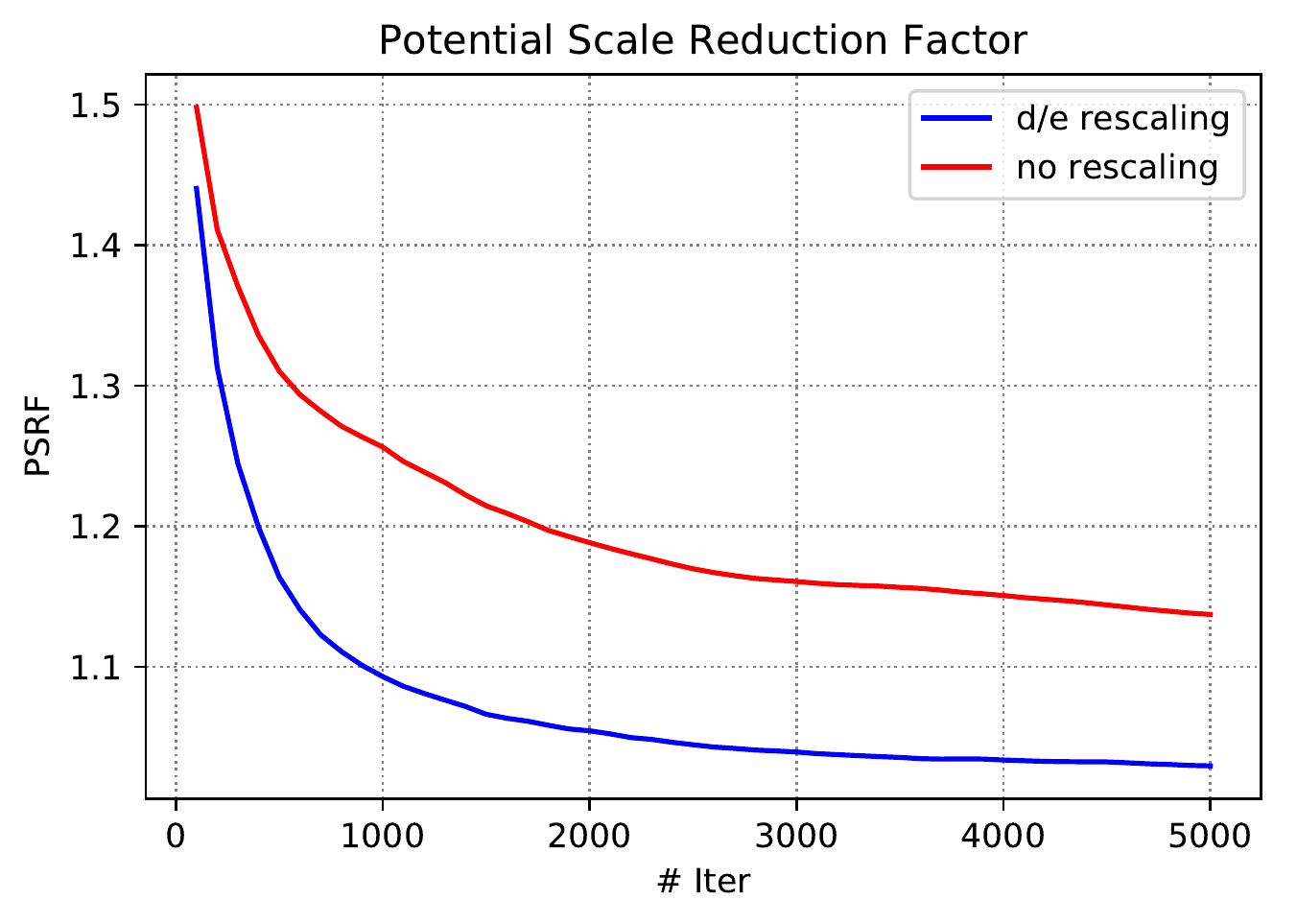}
  \caption{}
\end{subfigure}
\caption{(a,b) Empirical mixing time analysis for sampling a set of size at most $d = 40$ for varying ground set sizes, (a) the PSRF score for each set of chains, (b) the approximate mixing time obtained by thresholding at PSRF equal to $1.05$, (c) comparison of Algorithm \ref{alg: rescaled general SLC sampling MCMC} and a M-H algorithm where the proposal is built using $\mathcal{H}_{d} \nu$: $d= 100$ and $n=250$.}
\label{fig: PSRF and mix time varying n}
\end{figure}

Figure \ref{fig: PSRF and mix time varying d} considers the results of running the Metropolis-Hastings algorithm on a sequence of problems with different cardinality constraints $d$. In each case we considered the distribution $\nu(S) \propto \sqrt{\det( L_{S})} \mathbf{1}\{ \abs{S} \leq d \}$ where $L$ is a randomly generated $250 \times 250$ PSD matrix. Here $L_{S}$ denotes the $\abs{S} \times \abs{S}$ submatrix of $L$ whose indices belong to $S$. These simulations suggest that the mixing time grows linearly in $d$ for a fixed $n$.

Figure \ref{fig: PSRF and mix time varying n} considers the results of running the Metropolis-Hastings algorithm on a sequence of problems with different ground set sizes. In each case we considered the distribution $\nu(S) \propto \sqrt{\det( L_{S})} \mathbf{1}\{ \abs{S} \leq 40 \}$ where $L$ is a randomly generated PSD matrix where of appropriate size $n$. These simulations suggest that the mixing time grows sublinearly in $n$ for a fixed $d$.

It is important to know whether the mixing time is robust to different spectra $\sigma_L$ of $L$. We consider three cases, (i) smooth decay $\sigma_L = [n]$, (ii) a single large eigenvalue $\sigma_L = \{ n , (n-1)/2, (n-2)/2, \ldots, 2/2, 1/2\} $, and (iii) one fifth of the eigenvalues are equal to $n$, the  rest equal to $1/n$. Note that due to normalization, multiplying the spectrum by a constant does not affect the resulting distribution. The results for (i) are the content of  Figures \ref{fig: PSRF and mix time varying d} and \ref{fig: PSRF and mix time varying n} (a,b). Figures \ref{fig: PSRF and mix time varying d one big eval} and  \ref{fig: PSRF and mix time varying n one big eval} show the results for (ii) and figures \ref{fig: PSRF and mix time varying d step spectrum} and  \ref{fig: PSRF and mix time varying n step spectrum} show the results for (iii). Figures  \ref{fig: PSRF and mix time varying d one big eval}-\ref{fig: PSRF and mix time varying n step spectrum}  can be found in Appendix \ref{appendix: experiments}.

Finally, we address the question of why the proposal distribution was built using the particular choice of $\mu$ we made. Indeed one may use { \tt Base Exchange Walk } for \emph{any} homogenous distribution on $2^{[n]}$ to build a sampler, one simply needs to compute the appropriate acceptance probabilities. We restrict our attention to SLC distributions so as to be able to build on the recent mixing time results for homogenous SLC distributions. An obvious alternative to using $\mu$ to build the proposal is to use $\mathcal{H}_{d} \nu$. Figure \ref{fig: PSRF and mix time varying n}(c) compares the empirical mixing time of these two chains. The strong empirical improvement justifies our choice of adding the extra rescaling factor $d/e$.

 \vspace*{-5pt} 
\section{Discussion}

In this paper we introduced strongly log-concave distributions as a promising class of models for diversity. They have flexibility beyond that of strongly Rayleigh distributions, e.g., via exponentiated and cardinality constrained distributions (which do not preserve the SR property). We derived a suite of MCMC samplers for general SLC distributions and associated mixing time  bounds. For optimization, we showed that SLC distributions satisfy a weak submodularity property and proved mode finding guarantees. 

Still, many open problems remain. Although the mixing time bound has the interesting property of not directly depending on $n$, the $O(2^d)$ dependence seems quite conservative compared to the empirical mixing time results. An important future direction would be to bridge this gap. More fundamentally, the negative dependence properties of SLC distributions need to be explored in greater detail. Although in this work we proved a weak submodularity property for SLC distributions, we know of no examples of SLC distributions that are not log-submodular in the usual strong sense. This leads to the following conjecture, which if true would lead to stronger optimization guarantees.

\begin{Conjecture}
All strongly log-concave distributions are log-submodular.
\end{Conjecture}
 \vspace*{-5pt} 
Finally, in order for SLC models to be deployed in practice the user needs a way to learn a good SLC model from data. Both exponentiation and cardinality constraint add a single parameter that must be learned. We leave the question of how best to learn these parameters as an important topic for future work. 

{

  \bibliographystyle{plainnat}
  \setlength{\bibsep}{3pt}
  \bibliography{bibliography}
}

\clearpage
\appendix

\section{Proofs for operations preserving strong log-concavity}

In this section we prove Theorems \ref{thm,: weighted homogenization}, \ref{thm: weighted exponentiation}, and \ref{thm: closure under polarization}.

\subsection{Closure under scaled homogenization}

Let us begin this section by observing that closure under homogenization and symmetric homogenization both fail for strongly log-concave polynomials. The homogenization of a polynomial $f(z) = \sum_{ \abs{S} \leq d} c_S z^S$ is $f_{\text{h}}(z,y) = \sum_{\abs{S} \leq d} c_S z^S y^{d - \abs{S}}$, and its symmetric homogenization is $ f_{\text{sh} } = \Pi( f_{\text{h}})$.

We will use the following lemma.

\begin{lemma} \cite{anari2018log2}
$f(y,z) = a + by + cz + dyz$ with $a,b,c,d \in \mathbb{R}_+$ is SLC if and only if $2bc \geq ad$.
\end{lemma}

The counterexample is as follows: by the preceding lemma $f(y,z) = 1 + 2y + z + 3yz$ is SLC. Then note that its homogenization is $f_{\text{h}}(w,y,z) = w^2 + 2wy + wz + 3yz$. A quick computational check then shows that $\nabla^2(f_{\text{h}})$ has eigenvalues $-3.1, 0.4, 4.7$ each to one decimal place. Furthermore the symmetric homogenization of $f$ is,

\[ f_{\text{sh}}(w_1,w_2,y,z) = w_1w_2 + (w_1 + w_2)y + 1/2(w_1 + w_2)z + 3yz ,\]

and  one may check that $\nabla^2(f_{\text{sh}})$ has eigenvalues $3.8, -3.1, -1.0, 0.3$  to one decimal place. This shows that SLC is not closed under homogenization or symmetric homogenization. So we seek modified operations that are conserved by SLC. In Section 
\ref{section: preserving SLC} we introduced the rescaled homogenization of $f = \sum_{ \abs{S} \leq d} c_S z^S$,

\[\mathcal{H}_{k} f(z,y) = \sum_{\abs{S} \leq k} \frac{c_S}{(k - \abs{S})!} z^Sy^{k-\abs{S}}. \]

\begin{theorem}\label{conj: weighted homogenization}
Let $M = ( [n], \mathcal{I})$ be a rank $d$ matroid and $f = \sum_{S \in \mathcal{I}}c_S z^S \in \mathbb{R}[z_1, \ldots , z_n]$ be SLC where $c_S  > 0 $ for all $S \in \mathcal{I}$. For any $k \leq d$ the polynomial $\mathcal{H}_{k} f$ is SLC. 
\end{theorem}

A key component of proving this theorem is the following lemma.

\begin{lemma}\label{lemma: pertial deriv of f_wh is log-conc}
Let $f(z) = \sum_{ S \in \mathcal{I}} c_S z^S \in \mathbb{R}[z_1, \ldots , z_n]$ be multiaffine and SLC and suppose that $M = ([n], \mathcal{I}) $ is a matroid of rank $d$ and $k \leq d$. Then  $\partial_y^{k-2} (\mathcal{H}_{k} f)$ is log-concave.
\end{lemma}

\begin{proof}
Let $q =   \partial_y^{k-2} (\mathcal{H}_{k} f)$. We compute, 

\begin{align*}
  q(y,z) &= \partial_y^{k-2} \bigg ( \sum_{S \in \mathcal{I}_k} \frac{c_S}{(k-\abs{S})!}  z^S y^{k-\abs{S}} \bigg  ) \\
    &= \frac{1}{2}c_\varnothing y^2 + y  \sum_{  \{i\}\in \mathcal{I}_k} c_i z_i  + \sum_{  \{i,j\}\in \mathcal{I}_k} c_{i,j} z_i z_j 
\end{align*}

Let $Q = \nabla^2  q$. Note that since $ q $ is of degree two, $Q$ is in fact a constant and does not depend on $y$ or $z$. Therefore, $q$ is log-concave on $\mathbb{R}_{\geq 0}^{n+1}$ if and only if it is log-concave at $a = (1,0,\ldots,0)^\top $. This happens if and only if $(a^\top Q a)Q - (Qa)(Qa)^\top $ is negative semidefinite by Lemma \ref{lemma: log conc iff one pos eval}. But this is only true if and only if the matrix $[c_\varnothing c_{ij} - c_ic_j ]_{i,j \in [n]}$ is negative semidefinite. By definition

\[ \nabla^2 \log f = \frac{(\nabla^2 f) f - (\nabla f)(\nabla f)^\top}{f^2} \preccurlyeq 0\]

and evaluating at $z =0$ we notice that $f(0) = c_\varnothing$, $\partial_i f(0) = c_i$, and $\partial_{ij}f(0) = c_{ij}$.  So indeed we have that $[c_\varnothing c_{ij} - c_ic_j ]_{i,j \in [n]} \preccurlyeq 0$. 
\end{proof}

\begin{proof}[Proof of Theorem \ref{conj: weighted homogenization}]
To prove strong log-concavity we proceed by verifying the hypotheses of Theorem \ref{thm: anari log-concave characterization}. Let $\alpha \in \mathbb{N}^d$ and $m \in \mathbb{N}$ such that $\abs{\alpha } + m \leq k-2$. The first order of business is to show that $\partial^\alpha_z \partial^m_y (\mathcal{H}_{k} f)$ is indecomposable.  If $\alpha_i > 1$ for any $i$ then the expression equals $0$, so we may assume $\alpha = \mathbf{1}_K$ for some $K \subseteq [n]$. Then note that 

\begin{align*}
 \partial^K_z \partial^m_y (\mathcal{H}_{k} f) &= \partial^m_y  \sum_{S \in \mathcal{I}_k/K}\frac{c_{S \cup K}}{(k-\abs{S}-\abs{K})!} z^{S } y^{k-\abs{S} -\abs{K} } \\ 
 &= \partial^m_y  \sum_{S \in (\mathcal{I}/K)_{k - \abs{K}}}\frac{c_{S \cup K}}{(k-\abs{S}-\abs{K})!} z^{S } y^{k-\abs{S} -\abs{K} } \\
 &=: \partial^m_y  g,
 \end{align*}
 
 where $\mathcal{I}/K$ is the family of independent sets of $M/K$, the matroid contraction of $M$ by $K$. We first check indecomposability of $\partial^K_z \partial^m_y (\mathcal{H}_{k} f)$. Note that if $i \in [n] \setminus K$ is a loop of $M/K$ then the variable $z_i$ does not appear in $g$ and $\partial_i g = 0$. Similarly $\partial_i g= 0$ for all $i \in K$. Otherwise the monomial $z_i y^{k-1-m}$ appears in $\partial^m_y g$ with non-zero coefficient. Since $k-1-m \geq 1$ this implies that $\partial_i \partial_y g$ is non-zero. In particular the graph formed in the definition of indecomposability is a star centered at $y$ and therefore connected, proving that $ \partial^K_z \partial^m_y (\mathcal{H}_{k} f)$ is indecomposable.
 
Now suppose that $\abs{K} + m = k-2$. Notice that $\partial_z^K f = \partial_z^K \sum_{S \in \mathcal{I}} c_Sz^S = \sum_{S \in \mathcal{I}/K } c_{S \cup K} z^S$ is SLC, and

\[ \mathcal{H}_{k - \abs{K}}( \partial_z^K f) = \sum_{S \in ( \mathcal{I}/K)_{k - \abs{K}}} \frac{c_{S\cup K}}{(k - \abs{S} - \abs{K})!}z^S y^{k - \abs{S} - \abs{K}}.\]

So $\partial^K_z \partial^m_y (\mathcal{H}_{k} f) = \partial_y^m \mathcal{H}_{k - \abs{K}}( \partial_z^K f)  = \partial_y^{k - \abs{K} -2} \mathcal{H}_{k - \abs{K}}( \partial_z^K f)$ and we may apply Lemma \ref{lemma: pertial deriv of f_wh is log-conc} to conclude $\partial^K_z \partial^m_y (\mathcal{H}_{k} f) $ is log-concave.
\end{proof}

\subsection{Closure under scaled exponentiation}\label{sec: closure under exponentiation} 

For matrices $A = [a_{ij}]$ and $B = [b_{ij} ]$ and scalar $\alpha \in \mathbb{R}$ we write $A^{\circ \alpha}$  to denote the element-wise power $[a_{ij}^\alpha]$ and $A \circ B = [ a_{ij}b_{ij}]_{ij}$ to denote the Hadamard (element-wise) product. The proof of Theorem \ref{thm: weighted exponentiation} and of Theorem $1.7$ from \cite{anari2018log2} both boil down to the following linear algebra fact.

\begin{lemma}\footnote{This result was independently discovered by Br{\"a}nd{\'e}n and Huh \cite{branden2019lorentzian}. }\label{lemma: one pos eval implies exp has one pos eval}
Suppose $A= [a_{ij}]$ is symmetric, has non-negative entries and at most one positive eigenvalue. Then $A^{\circ \alpha}$also has at most one positive eigenvalue for $0 \leq \alpha \leq 1$.
\end{lemma}

To prove Lemma \ref{lemma: one pos eval implies exp has one pos eval} we recall a couple of of facts from linear algebra. We shall call a matrix $A$ conditionally negative definite if $z^\top A z \leq 0$ for all $z$ such that $z^\top \mathbf{1} =0$.

\begin{lemma}\label{lemma: one pos eval implies perron normalization is cnd}\cite{bapat1997nonnegative}
Suppose $A= [a_{ij}]$ is symmetric, has positive entries, and at most one positive eigenvalue. Then $ \bigg [ \frac{a_{ij} }{v_iv_j} \bigg ] $ is conditionally negative definite, where $v$ is the Perron-Frobenius eigenvector of $A$.
\end{lemma}

\begin{lemma}\label{lemma: cnd implies exp is cnd}\cite{berg1984harmonic}
Suppose $A \in \mathbb{R}^{n \times n}$ is conditionally negative definite. Then $A^{\circ \alpha}$ is conditionally negative definite for $0 \leq \alpha \leq 1$.
\end{lemma}

With these two facts in hand we are now ready to prove Lemma \ref{lemma: one pos eval implies exp has one pos eval}.

\begin{proof}[Proof of Lemma \ref{lemma: one pos eval implies exp has one pos eval}]
Assume that $a_{ij} > 0$ for all $i$ and $j$. The general case is then obtained by a limiting argument. Since $A= [a_{ij}]$ is symmetric, has positive entries, and at most one positive eigenvalue, Lemma \ref{lemma: one pos eval implies perron normalization is cnd} implies that $ \bigg [ \frac{a_{ij} }{v_iv_j} \bigg ] $ is conditionally negative definite, where $v$ is the  Perron-Frobenius eigenvector of $A$.  Then Lemma \ref{lemma: cnd implies exp is cnd} tell us that 

\[ B = \bigg [ \frac{a_{ij} ^\alpha}{v_i^\alpha v_j^\alpha} \bigg ] =  [a_{ij} ^\alpha] \circ ( v v^\top)^{ \circ - \alpha} = A^{\circ \alpha} \circ ( v v^\top)^{ \circ - \alpha} \]

is also conditionally negative definite. Note the identity,

\[ A^{\circ \alpha} = [a_{ij}^\alpha]  =  B \circ ( v v^\top)^{ \circ  \alpha} = \text{diag}( v ^{ \circ  \alpha} ) B \text{diag}( v ^{ \circ  \alpha} ).\]

Since the entries of the Perron-Frobenius eigenvector are all strictly positive, $\text{diag}( v ^{ \circ  \alpha} )$ is non-singular. We may therefore apply Sylvester's law of inertia to conclude that $B$ and $ A^{\circ \alpha} = \text{diag}( v ^{ \circ  \alpha} ) B \text{diag}( v ^{ \circ  \alpha} )$ have the same number of positive eigenvalues: one. 

\end{proof}

Next, we prove Theorem \ref{thm: weighted exponentiation} by showing how it reduces to exactly the statement of Lemma \ref{lemma: one pos eval implies exp has one pos eval}.

\begin{proof}[Proof of Theorem \ref{thm: weighted exponentiation}]
Consider $S \subseteq [n]$ and $m \in \mathbb{N}$ such that $\abs{S}  + m = d-2$. Denoting $\nabla^2( \partial^S_z \partial^m_z \mathcal{H}_{k}f) = A = [a_{ij}]_{i,j = 1}^{n+1}$ one observes that 

\[ \nabla^2( \partial^S_z \partial^m_z \mathcal{H}_{k, \alpha} f) = A^{\circ \alpha} = [a_{ij}^\alpha]_{i,j = 1}^{n+1}. \]

By Theorem \ref{thm,: weighted homogenization} the matrix $A$ has at most one positive eigenvalue and so we may apply Lemma \ref{lemma: one pos eval implies exp has one pos eval} to yield the result.
\end{proof}

Note Lemma \ref{lemma: cnd implies exp is cnd} also permits a simplified proof of the following theorem due to Anari et al. concerning the homogeneous case.

\begin{theorem}\label{thm: homog closur under exp}
Suppose $f = \sum_{ \abs{S} = d} c_S  z^S \in \mathbb{R}_+[z_1, \ldots , z_n]$ is SLC. Then $f_\alpha = \sum_{ \abs{S} = d} c_S^\alpha  z^S $ is SLC for any $0 \leq \alpha \leq 1$.
\end{theorem}

\begin{proof}

We prove strong log-concavity of $f_\alpha$ by verifying the hypotheses of Theorem \ref{thm: anari log-concave characterization}. Assume $c_S > 0$ for all $\abs{S} = d$. The general case is then obtained by taking point-wise limits. Let $S \subseteq [n]$ be such that $\abs{S} = d-2$. Then notice that 

\[ \nabla^2 (\partial^Tf ) = \big  [ c_{S\cup \{ i,j\} } \big ]_{ij} \hspace{0.5cm} \text{and} \hspace{0.5cm} \nabla^2 (\partial^Tf _\alpha) = \big [ c_{S\cup \{ i,j\}}^\alpha \big ]_{ij} .\]

Since $f$ is SLC, Lemma \ref{lemma: log conc iff one pos eval} implies $ \nabla^2 (\partial^Tf ) $ has at most one positive eigenvalue. So Lemma \ref{lemma: one pos eval implies exp has one pos eval} implies that $\nabla^2 (\partial^Tf _\alpha) $ also has at most one positive eigenvalue, which proves the strong log-concavity of $f_\alpha$ by applying Lemma \ref{lemma: log conc iff one pos eval} in the other direction.
\end{proof}

It is a reasonable question to ask whether or not the preceding theorem or Theorem \ref{thm: weighted exponentiation} can be extended to the regime $\alpha > 1$. This is in fact not the case. First we show that if either holds for \emph{any} $\alpha > 1$ then it must hold for \emph{all} $\alpha >1$, then we give a counterexample showing that it fails for $\alpha=2$ for both cases. Note carefully that the conclusion is therefore stronger than a mere existence claim. In fact we may conclude: Theorems \ref{thm: weighted exponentiation} and \ref{thm: homog closur under exp} both fail for \emph{all} $\alpha >1$.

To make the following statement succinct let us define $\mathcal{A}$ to be the set of all symmetric real-valued matrices with non-negative entries and at most one positive eigenvalue.

\begin{lemma}
 Suppose there is a $\alpha^* > 1$ such that: if $A \in \mathcal{A}$ then $A^{ \circ \alpha^*} \in \mathcal{A}$. Then for any $\alpha > 1$:  if $A \in \mathcal{A}$ then $A^{ \circ \alpha} \in \mathcal{A}$
\end{lemma}

\begin{proof}
Let $A \in \mathcal{A}$. We may repeatedly apply the hypothesis to conclude that $A^{ \circ m \alpha^*} \in \mathcal{A}$ for \emph{any} $m \in \mathbb{N}$. So in particular we may pick $m$ sufficiently big that $m \alpha^* > \alpha$. But now $\delta = \alpha / m \alpha^* <1$ so we may apply Lemma \ref{lemma: one pos eval implies exp has one pos eval} to conclude that $ A^{\circ \alpha} = A^{ \circ \delta m \alpha^*} \in \mathcal{A}$ .
\end{proof}

\begin{cor}
Suppose there is a $\alpha^* >1$ such that: if $f = \sum_{ \abs{S} = d} c_S  z^S \in \mathbb{R}_+[z_1, \ldots , z_n]$ is SLC, then  $f_{\alpha^*} = \sum_{ \abs{S} = d} c_S^{\alpha^*} z^S $ is SLC. Then for any $\alpha > 1$: if $f \in \mathbb{R}_+[z_1, \ldots , z_n]$ is SLC then $f_\alpha$ is SLC.
\end{cor}

\begin{cor}
Suppose there is a $\alpha^* >1$ such that: if $f = \sum_{ \abs{S} \leq d} c_S  z^S \in \mathbb{R}_+[z_1, \ldots , z_n]$ is SLC, then  $\mathcal{H}_{k, \alpha^*}f $ is SLC for $k \leq d$. Then for any $\alpha > 1$: if $f \in \mathbb{R}_+[z_1, \ldots , z_n]$ is SLC then $\mathcal{H}_{k, \alpha}f $ is SLC.
\end{cor}

For the counterexample, consider $f = 10wx + 3wy + 2wz + 2xy + 6xz$. The Hessian of $f$ equals 
\[ \nabla^2 f =
\begin{bmatrix}
    0 & 10 & 3 & 2\\
    10 & 0 & 2 & 6\\
    3 & 2 & 0 & 1\\
    2 & 6 & 1 & 0
\end{bmatrix}.
\]

One can numerically check that $\nabla^2 f$ has eigenvalues $\{1  13.6, -10.9, -2.2, -0.4\}$ to one decimal place, so $f$ is log-concave by Lemma \ref{lemma: log conc iff one pos eval} and hence SLC since it is of degree $2$. However, $f_2$ has Hessian equal to 

\[ \nabla^2 f_2 =
\begin{bmatrix}
    0 & 100 & 9 & 4\\
    100 & 0 & 4 & 36\\
    9 & 4 & 0 & 1\\
    4 & 36 & 1 & 0
\end{bmatrix}.
\]

which has eigenvalues $\{108.4, -105.2, -4.0, 0.8\}$ to one decimal place. So $f_2$ is not SLC.

The same example can be used to build a counterexample to Theorem \ref{thm: weighted exponentiation} in the regime $\alpha > 1$. Indeed setting $w=1$ in $f$ we obtain an SLC polynomial $g =  10x + 3y + 2z + 2xy + 6xz$ such that $\mathcal{H}_{2, 2} g = f_2 $ is not SLC.
\subsection{Closure under polarization}

We begin by observing an algebraic identity that allows one to push derivatives inside the polarization operation $\Pi$. 

\begin{lemma}
Let $f = \sum_{ \abs{S} \leq d} c_S  z^S y^{d-\abs{S}} \in \mathbb{R}[z_1, \ldots , z_n, y]$. Then $\partial_{y_{i}} \Pi (f) = \frac{1}{d} \Pi (\partial _y f)$  and $\partial_{z_{j}} \Pi (f) =  \Pi (\partial _{z_j} f)$  for $i \in [d]$ and $j \in [n]$.
\end{lemma}

\begin{proof}
Since $\Pi (f) \in \mathbb{R}[z_1, \ldots , z_n, y_1, \ldots , y_d]$ is symmetric in $y_1, \ldots, y_d$, to prove the $y_i$ part of the claim it suffices to prove the claim for $\partial_{y_{d} }\Pi (f)$ only. Recall that the polarization of $f$ is,

\[ \Pi (f)(z_1, \ldots , z_n , y_1, \ldots , y_d ) = \sum_{\abs{S} \leq d}  c_S  z^S  {d \choose \abs{S}}^{-1}  e_{d - \abs{S}}(y_1, \ldots , y_d) \]

where $e_k(y_1, \ldots , y_d)$ is the $k$th elementary symmetric polynomial in $d$ variables. We begin computing directly, 

\begin{align*}
    \partial_{y_d} \Pi (f) &= \partial_{y_d} \bigg ( \sum_{ \abs{S} \leq d}  c_S  z^S  {d \choose \abs{S}}^{-1}  e_{d - \abs{S}}(y_1, \ldots , y_d) \bigg ) \\
    &=  \sum_{\abs{S} \leq d}  c_S  z^S  {d \choose \abs{S}}^{-1}  \partial_{y_d} \big ( e_{d - \abs{S}}(y_1, \ldots , y_d) \big ) \\ 
    &=  \sum_{\abs{S} < d}  c_S  z^S  {d \choose \abs{S}}^{-1}  e_{d - \abs{S} - 1}(y_1, \ldots , y_{d-1}) \big ) \\ 
    &= \sum_{\abs{S} < d}  c_S  z^S \bigg ( \frac{1}{d} (d- \abs{S}) {d-1 \choose \abs{S}}^{-1} \bigg ) e_{d - \abs{S} - 1}(y_1, \ldots , y_{d-1}) \\
    &= \frac{1}{d} \sum_{\abs{S} < d} (d- \abs{S})  c_S  z^S  {d-1 \choose \abs{S}}^{-1} e_{d - \abs{S} - 1}(y_1, \ldots , y_{d-1}) \\
    &= \frac{1}{d} \Pi \bigg ( \sum_{\abs{S} < d} (d- \abs{S})  c_S  z^S y^{d-1 - \abs{S}} \bigg ) \\
    &= \frac{1}{d} \Pi \bigg ( \partial_y \bigg \{ \sum_{ \abs{S} \leq d}  c_S  z^S y^{d- \abs{S}}  \bigg \} \bigg ) \\
    &= \frac{1}{d} \Pi ( \partial _y f ) ,
\end{align*}

where we used the elementary relation 

\[ { d \choose \abs{S} } ^{-1} = \frac{1}{d} (d - \abs{S}) { d-1 \choose \abs{S }}^{-1}. \]

The $\partial_{z_{j}} \Pi (f)$ part of the  claim follows by a similar, but simpler, calculation. 
\end{proof}

\begin{cor}\label{corollary: differentiation under polarization}
Let $f = \sum_{ \abs{S} \leq d} c_S  z^S y^{d-\abs{S}} \in \mathbb{R}[z_1, \ldots , z_n, y]$. Then $\partial_z ^\alpha \partial^\beta_y \Pi (f) = c \Pi (\partial_z ^\alpha \partial^{\abs{\beta}}_y f)$ for any $\alpha, \beta \in \mathbb{N}^n$ where $c > 0$ is some constant.
\end{cor}

\begin{lemma}\label{lemma: f indecomposable if and only if Pi(f) is}
$f = \sum_{ \abs{S} \leq d} c_S  z^S y^{d-\abs{S}}$ is indecomposable if and only if $\Pi (f) $ is indecomposable. 
\end{lemma}

\begin{proof}
$f = 0$ if and only if $\Pi (f) = 0$ and if this is the case then we are done. So suppose $f$ is not identically $0$. Similarly if $c_S = 0 $ for all $\abs{S} < d$ we are done since $f = \Pi(f)$.  Let $\mathcal{Z} = \{ z_1, \ldots , z_n \}$ and $\mathcal{Y} = \{ y_1, \ldots , y_d \}$. Suppose that $\Pi (f) $ is  \emph{not} indecomposable. Then there exists partitions $\mathcal{Z}_1 \cup \mathcal{Z}_2 = \mathcal{Z}$ and $\mathcal{Y}_1 \cup \mathcal{Y}_2 = \mathcal{Y}$ such that $\Pi (f) = g_1 + g_1$ where neither $g_1$ nor $g_2$ are identically $0$ and $g_1$ depends only on the variables $\mathcal{Z}_1 \cup \mathcal{Y}_1$ and $g_2$ depends only on the variables $\mathcal{Z}_2 \cup \mathcal{Y}_2$. However we must have $\{ \mathcal{Y}_1, \mathcal{Y}_2 \} = \{ \varnothing, \mathcal{Y} \}$ since  we may pick an $S$ with $\abs{S} < d$ such that $c_S \neq 0$, and so $\Pi (f)$ contains the sum of monomials $c_S  z^S  {d \choose \abs{S}}^{-1}  e_{d - \abs{S}}(y_1, \ldots , y_d)$, which includes some terms containing $y_i$ and $y_j$ for any distinct $i,j \in [d]$. This observation is sufficient since all coefficients are non-negative and hence do not cancel each other out. Without loss of generality suppose $\mathcal{Y}_1 = \mathcal{Y}$ and $\mathcal{Y}_1 = \varnothing $. However, upon setting $y = y_1 = \ldots = y_d$ in $g_1$ we discover that we may write $f = g_1 + g_2$ where $g_1$ depends only on $\mathcal{Z}_1 \cup \{ y \}$ and $g_2$ depends only on $\mathcal{Z}_2$. In other words, $f$ is not indecomposable. 

Conversely, suppose $f$ is not indecomposable. Then we may write  $f = g_1 + g_2$ where $g_1$ depends only on $\mathcal{Z}_1 \cup \{ y \}$ and $g_2$ depends only on $\mathcal{Z}_2$ were $\mathcal{Z} = \mathcal{Z}_1 \cup \mathcal{Z} _2 $ is a partition. But, since polarization is a linear operator, this implies that $\Pi (f) $ may be decomposed into the sum of two non-zero polynomials, one depending only on $\mathcal{Z}_1 \cup \mathcal{Y}$ and the other on $\mathcal{Z}_2$. So $\Pi (f) $ is not indecomposable either. 
\end{proof}

We recall a fact from linear algebra that we shall use in the coming proof.

\begin{lemma}\label{lemma: PAP^top eigenvalue lemma}
Suppose $A \in \mathbb{R}^{n \times n}$ is symmetric and let $R \in \mathbb{R}^{m \times n}$. If $A$ has at most one positive eigenvalue, then $RAR^\top$ has at most one positive eigenvalue.
\end{lemma}

\begin{proof}[Proof of Theorem \ref{thm: closure under polarization}]
Suppose $\Pi(f)$ is SLC. Then by definition, upon setting $y_i= y$ for all $i$ we obtain $f$. Therefore $f$ is SLC since Lemma \ref{lemma: SLC closed under affine transformations} states that the SLC property is invariant under affine transformations of the coordinates. Conversely, suppose $f$ is SLC. We shall check that $\Pi (f)$ satisfies the hypotheses of Theorem \ref{thm: anari log-concave characterization}. Note that $d = \deg (f) = \deg ( \Pi(f) ) $. Take any $\alpha , \beta \in \mathbb{N}^n$ such that $\abs{\alpha} + \abs{\beta} \leq d-2$. Then by Corollary \ref{corollary: differentiation under polarization}, $\partial_z^\alpha \partial_y^\beta \Pi (f) = c \Pi (\partial_z ^\alpha \partial^{\abs{\beta}}_y f)$ for some constant $c$. Since $f$ is SLC, so is $\partial_z ^\alpha \partial^{\abs{\beta}}_y f$ and so by Theorem \ref{thm: support is set of bases of a matroid} its support is M-convex (see Theorem \ref{thm: support is set of bases of a matroid}  for definintion), which implies indecomposability. Hence by Lemma \ref{lemma: f indecomposable if and only if Pi(f) is} $\partial_z^\alpha \partial_y^\beta \Pi (f)$ is indecomposable too. Now suppose that $\abs{\alpha} + \abs{\beta} = d -2$. We must verify that $\partial_z^\alpha \partial_y^\beta \Pi (f) = c \Pi (\partial_z ^\alpha \partial^{\abs{\beta}}_y f)$ is log-concave. $g := \partial_z ^\alpha \partial^{\abs{\beta}}_y f $ is homogeneous, of degree $2$, and multiaffine in all except one coordinate. Hence $g$ is of the form 

\[ g = \sum_{i,j =1}^n c_{ij}z_iz_j + y \sum_{i=1}^n c_i z_i + c_\varnothing y^2. \]

Since $f$ is SLC, $g$ is log-concave. By Lemma \ref{lemma: log conc iff one pos eval} this implies that 

\[ \nabla^2 g = 
\begin{bmatrix}
    c_\varnothing       & \dots & c_j & \dots \\
    \vdots    &  & \vdots & & \\
    c_i     & \hdots & c_{ij} & \dots & \\
    \vdots &  & \vdots
\end{bmatrix}
\]

(a constant) has at most one positive eigenvalue. We can explicitly write 

\[ \Pi(g) = \sum_{i,j =1}^n c_{ij}z_iz_j + \frac{1}{2}(y_1 + y_2) \sum_{i=1}^n c_i z_i + c_\varnothing y_1 y_2. \]

and therefore

\[\nabla^2 \big ( \Pi(g) \big ) = 
\begin{bmatrix}
    0 & c_\varnothing & \dots & \frac{1}{2}c_j  & \dots \\
    c_\varnothing    & 0    & \dots & \frac{1}{2}c_j & \dots \\
    \vdots & \vdots    &  & \vdots & & \\
    \frac{1}{2}c_i & \frac{1}{2}c_i     & \hdots & c_{ij} & \dots & \\
    \vdots & \vdots &  & \vdots
\end{bmatrix}.
\]

 It suffices now to show that $\nabla^2 \big ( \Pi(g) \big ) $ also has at most one positive eigenvalue. 
Consider the $(n+1) \times n$ matrix 

\[ R = 
\begin{bmatrix}
    1/2       &     &   & & \\
    1/2       &  1 &  \ & & \\
                &     & \ddots  &  & \\
                &     &    & & \\
                &     &    &  & 1 \\
\end{bmatrix}
\]

where the entries along the specified diagonal are ones, and everywhere else is zeros. By Lemma \ref{lemma: PAP^top eigenvalue lemma}, the following $(n+1) \times (n+1)$ matrix also has at most one positive eigenvalue,

\[ R (\nabla^2 g) R^\top = 
\begin{bmatrix}
    \frac{1}{2} c_\varnothing & \frac{1}{2}c_\varnothing & \dots & \frac{1}{2}c_j  & \dots \\
    \frac{1}{2} c_\varnothing    & \frac{1}{2} c_\varnothing   & \dots & \frac{1}{2}c_j & \dots \\
    \vdots & \vdots    &  & \vdots & & \\
    \frac{1}{2}c_i & \frac{1}{2}c_i     & \hdots & c_{ij} & \dots & \\
    \vdots & \vdots &  & \vdots
\end{bmatrix}.
\]

Finally observe that $\nabla^2 \big ( \Pi(g) \big ) + uu^\top = R (\nabla^2 g )R^\top$ where $u = ( \sqrt{c_\varnothing/2} , - \sqrt{c_\varnothing/2}, 0, \ldots, 0)^\top$. Cauchy's Interlacing Theorem implies that $\nabla^2 \big ( \Pi(g) \big ) $ has at most one positive eigenvalue.
\end{proof}

\subsection{Closure under constraining to subsets of a given size }\label{section: Closure under constraining to subsets of a given size }

We state two propositions which when translated into probabilistic language say that if $\pi$ is SLC then the distribution proportional to $\pi(S) \mathbf{1} \{ \abs{S} = k \}$ is also SLC. 

\begin{lemma}\label{lemma: closure under fixing subsets to size d}
Let $f = \sum_{ \abs{\alpha} \leq d}c_\alpha z^\alpha \in \mathbb{R}_+[z_1, \ldots , z_n]$ be of  degree $d$ and be SLC. Then $f_d(z) = \sum_{ \abs{\alpha} = d}c_\alpha z^\alpha$ is also SLC.
\end{lemma}

\begin{cor}\label{cor: closure under fixing subsets to size k}
Let $f = \sum_{ \abs{\alpha} \leq d}c_\alpha z^\alpha \in \mathbb{R}_+[z_1, \ldots , z_n]$ be of  degree $d$ and be SLC. Then for any $k \leq d$, $f_k(z) = \sum_{ \abs{\alpha} = k}c_\alpha z^\alpha$ is also SLC.
\end{cor}

\begin{proof}[Proof of Lemma \ref{lemma: closure under fixing subsets to size d}]
Assume $c_\alpha > 0 $ for all $\abs{\alpha} \leq d$. The general result is then obtained by taking point-wise limits of coefficients. So for any $\abs{\alpha} \leq d-2$, the polynomial $\partial^\alpha f_d $ is evidently indecomposable. Now assume $\abs{\alpha} = d-2$. Then note that $\nabla^2(\partial^\alpha f) = \nabla^2 ( \partial^\alpha f_d)$. Since $f$ is strongly log-concave $\nabla^2(\partial^\alpha f) $ has at most one positive eigenvalue, and hence so does $\nabla^2 ( \partial^\alpha f_d)$.
\end{proof}

\begin{proof}[Proof of Corollary \ref{cor: closure under fixing subsets to size k}]
This result immediately follows from combining Theorem \ref{thm,: weighted homogenization} and Lemma \ref{lemma: closure under fixing subsets to size d}.
\end{proof}

\section{The Metropolis-Hastings chain's stationary distribution and mixing time}

\subsection{The chain has the right stationary distribution}
We only need to check that the acceptance probability stated does indeed yield a chain with stationary distribution $\nu_{\text{sh}}$. We build our algorithm using the usual Metropolis-Hastings procedure. We consider the proposal distributions $Q$, each being Anari's base exchange kernel for the distribution $\mu$. The Metropolis-Hastings algorithm is as follows,
\begin{enumerate}
	\item Suppose the current state is $S \subseteq [n+d]$ with $\abs{S \cap [n] } = k$,
	\item Sample $T \sim Q(S, \cdot)$,
	\item Compute the acceptance probability $a = \min \bigg \{ 1, \frac{\nu_{\text{sh}}(T)}{\nu_{\text{sh}}(S)}\frac{Q(T,S)}{Q(S,T)} \bigg  \}$,
	\item With probability $a$, update $S \gets T$,
	\item Otherwise do not update.
\end{enumerate}

All that remains is to compute the acceptance probability. If $S=T $ then clearly $a=1$ and Algorithm \ref{alg: rescaled general SLC sampling MCMC} is in agreement. So suppose from now on that $S \neq T$. If further we have $\abs{S \cap T} < d-1$ then $Q(S,T) = 0$ and so we may discount this possibility since the proposal distribution will never sample such a $T$. This leaves only the case that $\abs{S \cap T} = d-1$. We may write the transition kernel explicitly

\[ Q(S,T) =  \frac{1}{d}  \frac{ \mu  (T) }{w( S \cap T)} \propto \frac{1}{d}  \bigg (\frac{d}{e} \bigg )^{ d - \abs{T \cap [n]}} \frac{ \nu_{ \text{swh} } (T) }{w( S \cap T)} \]

where we define $w( S \cap T) = \sum\limits_{ i \in  (S \cap T)^\mathsf{c} } \mu ( (S \cap T) \cup i ) $. Computing the ratio

\begin{align*}
\frac{Q(T,S)}{Q(S,T)} &= \frac{    1/d \cdot \mu (S) / w( S \cap T) } {    1/d \cdot \mu (T) / w( S \cap T) } \\
				&=  \frac{   \mu (S)  } {  \mu (T)  } \\
				&= \frac{ ( d - \abs{ T \cap [n]  })!}{ ( d - \abs{ S \cap [n] } )!}  \frac{(d/e)^{d - \abs{S \cap [n] }}} {(d/e)^{d - \abs{T \cap [n] }}}\frac{    \nu_{ \text{sh}} (S)  } {  \nu_{ \text{sh}} (T)  } \\
				&= \bigg ( \frac{d}{e} \bigg )^{ \abs{T \cap [n] } -k} \frac{ ( d - \abs{ T \cap [n]  })!}{ ( d -k)!} \frac{    \nu_{ \text{sh}} (S)  } {  \nu_{ \text{sh}} (T)  } ,\\
\end{align*}

and so, 

\begin{align*}
a &= \min  \bigg \{ 1,  \frac{    \nu_{ \text{sh}} (T)  } {  \nu_{ \text{sh}} (S) } \bigg ( \frac{d}{e} \bigg )^{ \abs{T \cap [n] } -k}   \frac{ ( d - \abs{ T \cap [n]  })!}{ ( d -k)!} \frac{    \nu_{ \text{sh}} (S)  } {  \nu_{ \text{sh}} (T)  } \bigg  \} \\
   &= \min \bigg  \{ 1 , \bigg ( \frac{d}{e} \bigg )^{ \abs{T \cap [n] } -k}   \frac{ ( d - \abs{ T \cap [n]  })!}{ ( d -k)!}  \bigg  \}.
\end{align*}

This leaves the following three cases,
\begin{enumerate}
	\item If  $\abs{ T \cap [n] } = k-1$, then $a = \min \{ 1, \frac{e}{d}  (d-k+1)  \} $. 
	\item If  $\abs{ T \cap [n] } = k$, then $a = \min \{ 1, 1 \} = 1$. 
	\item If  $\abs{ T \cap [n] } = k+1$, then $a = \min \{ 1, \frac{d}{e} \frac{1}{(d-k)} \} $. 
\end{enumerate}

These are exactly the acceptance probabilities given in Algorithm \ref{alg: rescaled general SLC sampling MCMC}.
\subsection{Mixing time bounds}

The mixing time of a chain can be bounded using an important quantity known as the log-Sobolev constant. A famous theorem due to Diaconis and Saloff-Coste \cite{diaconis1996logarithmic} shows that if $(Q,\pi)$ has log-Sobolev constant $\alpha$, then the mixing time of this chain is bounded by $t_{S_0}(\varepsilon) \leq \frac{1}{\alpha} \big ( \log \log \frac{1}{\pi(S_0)} + \log \frac{1}{2 \varepsilon^2} \big )$. For a particular problem instance the objective is therefore to find a lower bound on $\alpha$. 

Let $P$ denote the transition kernel described in Algorithm \ref{alg: rescaled general SLC sampling MCMC} that we have just confirmed have stationary distribution $\nu_{ \text{sh} }$. In this section we shall prove Theorem \ref{thm: MCMC mixing time bound rescaled}, a mixing time bound on the chain $(P, \nu_{ \text{sh} })$ . The bound is obtained by combining two pieces of information: the fact,obtained by Cryan \cite{cryan2019modified}, that the chain $(Q, \nu_{ \text{swh} } )$ has log-Sobolev constant bounded below by $1/d$, and a theorem due to Diaconis and Saloff-Coste \cite{diaconis1996logarithmic} that allows one to compare the log-Sobolev constants of two chains. For the statement of the comparison theorem see appendix \ref{appendix: supplementary}.

Recall that $P$  is constructed using the base exchange walk for the following distribution

\begin{equation*}
\mu(S) = \frac{1}{Z} \frac{1}{ ( d - \abs{S \cap [n]})!} \bigg (\frac{d}{e} \bigg )^{ d - \abs{S \cap [n]}} \nu_{ \text{sh} }(S) 
\end{equation*}

for $S \subseteq [n+d]$ where $Z$ is the partition function of $\mu$.

\begin{lemma} \label{lemma: bound on v_sh / mu}

\begin{equation*}
\frac{\sqrt{2 \pi} }{2^d }  Z \leq  \frac{ \nu_{ \text{sh} } (S) }{ \mu(S) }  \leq e \sqrt{d}Z
\end{equation*}

for all $S \subseteq [n+d]$.
\end{lemma}

\begin{proof}
For this proof set $k = \abs{S \cap [n]}$. Then
\begin{equation*}\label{eqn: v_sh / mu}
\frac{1}{Z} \frac{ \nu_{ \text{sh} } (S) }{ \mu(S) }  =   ( d -k)!  \bigg (\frac{e}{d} \bigg )^{ d -k}.
\end{equation*}

Stirling's approximation says that

\begin{equation*}
\sqrt{2 \pi} \sqrt{d-k} \bigg ( \frac{d-k}{e} \bigg )^{d-k} \leq (d-k)! \leq e \sqrt{d-k} \bigg ( \frac{d-k}{e} \bigg )^{d-k}
\end{equation*}

which upon substitution into (\ref{eqn: v_sh / mu}) yields

\begin{equation*}
\sqrt{2 \pi} \sqrt{d-k} \bigg ( \frac{d-k}{d} \bigg )^{d-k} \leq  (d-k)!  \bigg (\frac{e}{d} \bigg )^{ d -k}\leq e \sqrt{d-k} \bigg ( \frac{d-k}{d} \bigg )^{d-k}.
\end{equation*}

Inserting $k=0$ one observes that $e \sqrt{d}$ is a tight upper bound. It remains only to find a lower bound. Note that if $k=d$ then $\frac{1}{Z} \frac{ \nu_{ \text{sh} } (S) }{ \mu(S) } =1$. So we may assume that $k \in [d-1]$. We may bound below as follows,

\begin{align*}
\sqrt{2 \pi} \sqrt{d-k} \bigg ( \frac{d-k}{d} \bigg )^{d-k} &\geq \sqrt{2 \pi}\bigg ( \frac{d-k}{d} \bigg )^{d-k} \\
	& =  \sqrt{2 \pi}\bigg ( \frac{d}{d-k} \bigg )^{-(d-k)} \\
	&=  \sqrt{2 \pi}\bigg ( 1 + \frac{k}{d-k} \bigg )^{-(d-k)} \\
	&\geq  \min_{ k_1, k_2 \in [d-1] } \sqrt{2 \pi}\bigg ( 1 + \frac{k_1}{d-k_2} \bigg )^{-(d-k_2)} \\
	&\geq  \min_{ k_1 \in [d-1] } \sqrt{2 \pi}\bigg ( 1 + \frac{k_1}{d} \bigg )^{-d} \\
	&\geq   \sqrt{2 \pi}\bigg ( 1 + \frac{d}{d} \bigg )^{-d} \\
	&= \sqrt{2 \pi} 2^{-d}.
\end{align*}

where the penultimate inquality is due to the fact that $(1 + x/n)^n$ is an increasing function of $n$ for any fixed $x \in \mathbb{R}$.
\end{proof}

\begin{proof}[ Proof of  Theorem \ref{thm: MCMC mixing time bound rescaled}]

We seek $a,A > 0$ like in the statement of Theorem \ref{thm: Diaconis comparting log-sobolev}. Lemma \ref{lemma: bound on v_sh / mu} says exactly that for all $S \subseteq [n+d]$,
\begin{equation*}
\frac{\sqrt{2 \pi} }{2^d }  Z \leq  \frac{ \nu_{ \text{sh} } (S) }{ \mu(S) }  \leq e \sqrt{d}Z
\end{equation*}

where here $Z$ is the partition function of $\mu$.  So we may take $a =  e \sqrt{d}Z$. Next, recall that $P$  denotes the transition kernel defined by Algorithm \ref{alg: rescaled general SLC sampling MCMC}. We may write $P$ explicitly

\begin{equation*}
    P(S,T)  =
    \begin{cases*}
      a(S,T) Q(S,T) & if $\abs{T} = k-1$ \\
       Q(S,T) & if $\abs{T} = k$ and $T \neq S$ \\
         Q(S,T) + \sum\limits_{ \abs{R} = k \pm 1} (1 - a(S,R) )Q(S,R) & if $T=S$ \\
           a(S,T)Q(S,T) & if $\abs{T} = k+1,$ 
    \end{cases*}
  \end{equation*}

where $a(S,T)$ denotes the previously defined acceptance probability for a proposed move to $T$ if the current state is $S$. Observe that one divided by the acceptance probability is bounded above by $\max \{ e, d/e \} = d/e$ for $d \geq 8$.
 So we see that $Q \leq \frac{d}{e} P$. Combining this with the lower bound on $\nu_{ \text{sh} }(S) /\mu(S) $ we find that we may take 
 
  \begin{equation*}
A = \frac{d}{e} \frac{2^d}{\sqrt{2\pi} Z}.
  \end{equation*}
 
 Applying Theorem \ref{thm: Diaconis comparting log-sobolev} and using the fact that the log-Sobolev constant of $(Q,  \nu_{ \text{swh} })$ is bounded below by $1/d$ we obtain the following lower bound on the log-Sobolev constant for $(P,  \nu_{ \text{sh} })$,
 
 \begin{equation*}
 \frac{1}{aA} \frac{1}{d} = \frac{ e \sqrt{2\pi} }{d^{5/2} 2^d}.
\end{equation*}

\end{proof}

\section{Proof of weak log-submodular properties}\label{appendix: submodular}

In this section we prove Theorems \ref{thm: weak submodular property}, \ref{thm: distorted greedy}, Corollary \ref{corollary: weak log-submodular distorted greedy} and give two other simple greedy algorithms for non-negative weakly log-submodular functions and monotone weakly log-submodular functions respectively. Whilst we only derive results for three algorithms here, we expect that many submodular maximization algorithms will have weak submodular analogues, including guarantees.

In this section we shall use the notation $\nu( e \mid S) = \nu( S \cup \{ e \}) - \nu(  S)$ for any set function $\nu : 2^{[n]} \rightarrow \mathbb{R}_+$ and any $S \subseteq [n]$ and $i \in [n]$.

\subsection{Functions whose weighted homogenized generating polynomial is SLC are weakly log-submodular}

\begin{proof}[Proof of Theorem \ref{thm: weak submodular property}]
By Theorem \ref{thm,: weighted homogenization}, the polynomial $g = \mathcal{H}_{n} f$ is  SLC where $f $ is the generating polynomial of $\nu$. Hence so is $\partial_z^S g $ for any $S \subseteq [n]$.  Applying Theorem \ref{thm: Huh weak log-submodularity results} with $z=0$ and $y=1$ we obtain,
\[ \partial_z^S g ( 0,1) \partial_i \partial_j \partial_z^S g (0,1 )  \leq 2 \bigg ( 1 - \frac{1}{d } \bigg ) \partial_i \partial_z^S g ( 0 ,1)  \partial_j \partial_z^S g ( 0,1) .\]

Rewriting this in terms of $\nu$ gives,

\[ \frac{\nu(S)}{(n - \abs{S})!} \frac{\nu(S \cup \{ i,j \})}{(n - \abs{S} - 2)!} \leq 2 \bigg ( 1- \frac{1}{d}  \bigg )  \frac{\nu(S \cup i )}{(n - \abs{S} - 1)!}  \frac{\nu(S \cup j )}{(n - \abs{S} - 1)!},\] 

which rearranges to,

\[ \nu(S) \nu(S \cup \{ i,j \}) \leq 2 \bigg ( 1- \frac{1}{d}  \bigg ) \bigg ( \frac{n - \abs{S}}{ n - \abs{S} -1} \bigg ) \nu(S \cup  i ) \nu(S \cup j).\]

Finally note that $\frac{n - \abs{S}}{n- \abs{S} -1} \leq 2$ since $n - \abs{S} \geq 2$.
\end{proof}

\subsection{Distorted greedy guarantees}

We include results from \cite{harshaw2019submodular} in the interests of completeness. Recall the definition $\Phi_i(S) = (1 - 1/k)^{k-i} \eta(S) - c(S)$. We also introduce the following object that will be useful for our analysis,

\[ \Psi_i(S,e) = \max \bigg \{ 0, \bigg ( 1- \frac{1}{k})^{k - (i+1)} \eta(e  \mid S) - c_e \bigg \}\]

for $i = 0, \ldots k-1$. Finally, note that by writing $\nu(\text{OPT}) - \nu(S)$ as a telescoping sum and using the definition of weak submodularity one finds that, 

\begin{equation*}
\nu(\text{OPT}) - \nu( S) \leq \sum_{e \in \text{OPT}} \nu( e \mid S) + \frac{1}{2} \ell ( \ell -1 ) \gamma.
\end{equation*}

This fact will be used in the proof of Lemma \ref{lemma: psi bound}. We prepare for the proof of Theorem \ref{thm: distorted greedy} by recalling two lemmas.

\begin{lemma}\cite{harshaw2019submodular}\label{lemma: phi increment}
In each iteration of the distorted greedy Algorithm \ref{alg: distorted greedy weak submodular maximization} we have,

\[ \Phi_{i+1}(S_{i+1}) - \Phi_i(S_i) = \Psi_i(S_i,e_i) + \frac{1}{k} \bigg ( 1 - \frac{1}{k} \bigg ) ^{k - (i+1)} \eta(S_i).\]
\end{lemma}

\begin{proof}
See \cite{harshaw2019submodular}.
\end{proof}

\begin{lemma}\label{lemma: psi bound}
In each iteration of the distorted greedy Algorithm \ref{alg: distorted greedy weak submodular maximization} we have,

\[ \Psi_i(S_i,e_i) \geq \frac{1}{k} \bigg ( 1 - \frac{1}{k} \bigg ) ^ {k - (i+1)} \big ( \eta(\text{OPT})  - \eta(S_i)  - \frac{1}{2} \ell (\ell -1) \gamma \big ) - \frac{1}{k} c(\text{OPT}) \]
\end{lemma}

\begin{proof}[Proof of \ref{lemma: psi bound} (adapted from \cite{harshaw2019submodular})]
\begin{align}
k \cdot \Psi_i(S_i,e_i) &= k \max_{e \in [n]} \bigg \{ 0,  \big (1- 1/k \big )^{k - (i+1)} \eta(e  \mid S) - c_e \bigg \} \\
&\geq \abs{\text{OPT}} \max_{e \in \text{OPT}} \bigg \{ 0,  \big (1- 1/k\big )^{k - (i+1)} \eta(e  \mid S) - c_e \bigg \} \label{eqn: OPT leq k} \\
&\geq  \sum_{e \in \text{OPT}} \bigg ( \big (1- 1/k \big )^{k - (i+1)} \eta(e  \mid S) - c_e \bigg ) \label{eqn: average less max} \\
&= \big (1- 1/k \big )^{k - (i+1)}   \sum_{e \in \text{OPT}}  \eta(e  \mid S) - c(\text{OPT}) \\
&\geq \big (1- 1/k \big )^{k - (i+1)} \big ( \eta(\text{OPT})  - \eta(S_i)  - \frac{1}{2} \ell (\ell -1) \gamma \big ) - c(\text{OPT}) \label{eqn: weak submod}.
\end{align}

Inequality (\ref{eqn: OPT leq k}) follows since $\abs{\text{OPT}} \leq k$ and restricting the domain of the maximum can only make the expression smaller, (\ref{eqn: average less max}) follows since the average of a collection of numbers is no bigger than the maximum, and (\ref{eqn: weak submod}) follows from the $\gamma$-weak submodularity of $\nu$.
\end{proof}

\begin{proof}[Proof of Theorem \ref{thm: distorted greedy} (adapted from \cite{harshaw2019submodular})]
Observe that 

\[ \Phi_0(S_0) = \bigg ( 1- \frac{1}{k} \bigg ) ^k \eta( \varnothing) - c( \varnothing)  =0 , \]

and 
\[ \Phi_k(S_k) = \bigg ( 1- \frac{1}{k} \bigg ) ^0 \eta(R) - c( R) = \eta(R) - c( R) . \]
Now,

\begin{align}
\nu(R) = \eta(R) - c( R) &=  \Phi_k(S_k) - \Phi_0(S_0)   \nonumber \\
&=\sum_{i=0}^{k-1}  \bigg ( \Phi_{i+1}(S_{i+1}) - \Phi_i(S_i) \bigg ). \label{eqn: telescoping sum psi}
\end{align}

Applying Lemma \ref{lemma: phi increment} then \ref{lemma: psi bound} we find,
\begin{align*}
\Phi_{i+1}(S_{i+1}) - \Phi_i(S_i) &=  \Psi_i(S_i,e_i) + \frac{1}{k} \bigg ( 1 - \frac{1}{k} \bigg ) ^{k - (i+1)} \eta(S_i) \\
& \geq \frac{1}{k} \bigg ( 1 - \frac{1}{k} \bigg ) ^ {k - (i+1)} \big ( \eta(\text{OPT})  - \eta(S_i)  - \frac{1}{2} \ell (\ell -1) \gamma \big ) - \frac{1}{k} c(\text{OPT}) \\
&+ \frac{1}{k} \bigg ( 1 - \frac{1}{k} \bigg ) ^{k - (i+1)} \eta(S_i)  \\
&\geq \frac{1}{k} \bigg ( 1 - \frac{1}{k} \bigg ) ^ {k - (i+1)} \big ( \eta(\text{OPT})  - \frac{1}{2} \ell (\ell -1) \gamma \big ) - \frac{1}{k} c(\text{OPT}) \\
\end{align*}

Which, upon inserting into equation (\ref{eqn: telescoping sum psi}) yields,

\begin{align*}
\nu(R) &= \sum_{i=0}^{k-1} \bigg \{ \frac{1}{k} \bigg ( 1 - \frac{1}{k} \bigg ) ^ {k - (i+1)} \big ( \eta(\text{OPT})  - \frac{1}{2} \ell (\ell -1) \gamma \big ) - \frac{1}{k} c(\text{OPT}) \bigg \} \\
&= \bigg ( 1 - \big ( 1- \frac{1}{k} \big ) ^k \bigg )  \big ( \eta(\text{OPT})  - \frac{1}{2} \ell (\ell -1) \gamma \big ) - c(\text{OPT})  \\
&\geq (1 - 1/e)  \big ( \eta(\text{OPT})  - \frac{1}{2} \ell (\ell -1) \gamma \big ) - c(\text{OPT})s.
\end{align*}
\end{proof}

\subsection{Double greedy guarantees}

In this section we introduce a modified double greedy unconstrained maximization algorithm and give theoretical guarantees in the case that $\nu$ is weakly submodular and non-negative. 

\begin{algorithm}
\caption{Double greedy weak submodular maximization of $\nu$}\label{alg: double greedy weak submodular maximization}
\begin{algorithmic}[1]
\State Let $X_0 = \varnothing$ and $Y_0 = [n]$
\For{ $i = 1, \ldots , n$}
\State Set $a_i = \max \{ \nu(X_{i-1} \cup i ) - \nu(X_{i-1}  ) + (n-i)\gamma, 0 \}$
\State Set $b_i = \max \{ \nu(Y_{i-1} \setminus i ) - \nu(Y_{i-1}  ) + (n-i)\gamma, 0 \}$
\State Sample $u \in [0,1]$ uniformly
\If {$u < a_i/(a_i + b_i)$} $X_i \gets X_{i-1} \cup  i$, and $Y_i \gets Y_{i-1}$
\Else \hspace{0cm} {$ X_i \gets X_{i-1} $, and $Y_i \gets Y_{i-1} \setminus i$}
\EndIf
\EndFor
\State \textbf{return} $X_n = Y_n$
\end{algorithmic}
\end{algorithm}

Next we give a double greedy unconstrained optimization algorithm. We shall use the convention that $a_i/(a_i + b_i) = 0$ whenever $a_i = b_i =0$. We denote an element of the set of maximizers $\arg \max_S \nu(S)$ by $\text{OPT}$. It is clear that Algorithm \ref{alg: double greedy weak submodular maximization} runs in linear time in $n$. It will be useful to notice that by repeatedly applying the definition of weak submodularity one finds that $\nu(T \cup i  ) - \nu(T)  \leq  \abs{T \setminus S} \gamma + \nu(S \cup i  ) - \nu(S)$ for any $S \subseteq T$.

\begin{theorem}\label{thm: double greedy}
Suppose $\nu : 2^{[n]} \rightarrow \mathbb{R}_+$ is $\gamma$-weakly submodular. Then the solution obtained by the double greedy algorithm $X_n = Y_n$ satisfies

\[ \mathbb{E}[ \nu(X_n)] \geq \frac{1}{2} \nu(\text{OPT}) - \frac{3}{16}n(n-1) \gamma  = \frac{1}{2} \nu(\text{OPT}) - O(n^2) . \]
\end{theorem}

The following lemma contains the bulk of the work required to prove Theorem \ref{thm: double greedy}. The proof is a generalization to the weak submodular setting of the proof of the original double greedy $1/2$-approximation theorem \cite{buchbinder2015tight} for submodular functions. 

\begin{lemma}\label{lemma: weak submodular building block lemma}
Let $\nu : 2^{[n]} \rightarrow \mathbb{R}_+$ be $\gamma$-weakly submodular. Then for $i = 1 , \ldots , n$ we have,

\[ \mathbb{E}[\nu ( \text{OPT}_{i-1}) - \nu ( \text{OPT}_{i})] \leq \frac{1}{2} \mathbb{E}[ \nu(X_i) - \nu(X_{i-1}) + \nu(Y_i)- \nu(Y_{i-1})] + \frac{3}{4}(n-i)\gamma \]
\end{lemma}

\begin{proof}
It suffices to prove the claim conditioning on the event that $X_{i-1} = A_{i-1} $ for any $A_{i-1}$ such that $X_{i-1} = A_{i-1} $ has non-zero probability. In particular therefore $A_{i-1} \subseteq [i-1]$. Fixing such an event we shall implicitly assume we have conditioned on this event for the rest of the proof. This means that the variables $X_{i-1}, Y_{i-1}, \text{OPT}_{i-1}, a_i$, and $b_i$ all become deterministic. 
First note that 

\begin{align*}
     \mathbb{E}[ \nu(X_{i} ) - \nu(X_{i-1})+ \nu(Y_{i} )-\nu(Y_{i-1} ) )] + (n-i)\gamma /2 \\
      &= \frac{a_i}{a_i + b_i} \bigg ( \nu(X_{i-1}\cup i ) - \nu(X_{i-1} )+ (n-i)\gamma/2 \bigg ) \\
  &+ \frac{b_i}{a_i + b_i} \bigg ( \nu(Y_{i-1}\setminus i ) - \nu(Y_{i-1}) + (n-i)\gamma /2 \bigg ) \\
  &= \frac{a_i^2 + b_i^2}{a_i + b_i}
\end{align*}

where the first equality follows by definition of the expectation. To see why the second equality holds first recall that due to $\gamma$-weak submodularity of $\nu$ we have, 
\begin{equation*}\label{eqn: weak submod expanded}
     \nu(X_{i-1} \cup i ) - \nu(X_{i-1})+ \nu(Y_{i-1} \setminus i )-\nu(Y_{i-1}) + (n-i) \gamma \geq 0,
\end{equation*}

where the coefficient in front of $\gamma$ is obtained by noting that $\abs{(Y_{i-1} \setminus i) \setminus X_{i-1}} = \abs{(Y_{i-1}  \setminus X_{i-1} \cup i} = n-i$. Now suppose that $\nu(X_{i-1}\cup i ) - \nu(X_{i-1} )+ (n-i)\gamma /2 < 0$. Then equation \ref{eqn: weak submod expanded} implies that we must have $\nu(Y_{i-1} \setminus i )-\nu(Y_{i-1}) + (n-i) \gamma /2 > 0$. This implies that $a_i =0$, $b_i > 0$ and the claimed equality is then readily seen to hold. Alternatively, if $\nu(Y_{i-1} \setminus i )-\nu(Y_{i-1}) + (n-i) \gamma /2 < 0$ . Then  $\nu(X_{i-1}\cup i ) - \nu(X_{i-1} )+ (n-i)\gamma /2 > 0$. This implies that $a_i >0$, $b_i =0 $ and once again the claimed equality holds. The remaining case, where both expressions are non-negative, is the simple case and also holds. 

Now let us bound the left hand side expression. We consider two cases. First suppose that $i \notin \text{OPT}$. Then $\text{OPT}_i = \text{OPT}_{i-1}\cup i$ if the updates were  $X_i \gets X_{i-1} \cup  i$, and $Y_i \gets Y_{i-1}$ and $\text{OPT}_i = \text{OPT}_{i-1}$ if the updates were $ X_i \gets X_{i-1} $, and $Y_i \gets Y_{i-1} \setminus i$. Hence,

\begin{align*}
    \mathbb{E}[\nu ( \text{OPT}_{i-1}) - \nu ( \text{OPT}_{i})] 
    &= \frac{a_i}{a_i + b_i} [\nu ( \text{OPT}_{i-1}) - \nu ( \text{OPT}_{i-1}\cup i)] \\
    &\leq  \frac{a_i}{a_i + b_i} [\nu (Y_{i-1}\setminus i) - \nu ( Y_{i-1}) + (n-i)\gamma] \\
    &\leq 
    \frac{a_ib_i}{a_i + b_i} + (n-i)\gamma /2.
\end{align*}

The first inequality follows due the weak submodularity and the fact that $(Y_{i-1}\setminus i ) \setminus \text{OPT}_{i-1} = Y_{i-1}\setminus ( \text{OPT}_{i-1} \cup i) \subseteq Y_{i-1}\setminus ( X_{i-1} \cup i) $ which has cardinality $n-i$. The second simply follows from the fact that  $\nu (Y_{i-1}\setminus i) - \nu ( Y_{i-1}) + (n-i)\gamma/2 \leq b_i$ and $a_i/(a_i + b_i) \leq 1$.

Alternatively, suppose $i \in \text{OPT}$. Then $\text{OPT}_i = \text{OPT}_{i-1} \setminus i$ if the updates were $ X_i \gets X_{i-1} $, and $Y_i \gets Y_{i-1} \setminus i$ and $\text{OPT}_i = \text{OPT}_{i-1}$ if the updates were $X_i \gets X_{i-1} \cup  i$, and $Y_i \gets Y_{i-1}$. So, 

\begin{align*}
    \mathbb{E}[\nu ( \text{OPT}_{i-1}) - \nu ( \text{OPT}_{i})] &= \frac{b_i}{a_i + b_i} [\nu ( \text{OPT}_{i-1}) - \nu ( \text{OPT}_{i-1}\setminus i)] \\
    &\leq  \frac{b_i}{a_i + b_i} [\nu (X_{i-1} \cup i) - \nu ( X_{i-1}) + (n-i)\gamma] \\
    &\leq 
    \frac{a_ib_i}{a_i + b_i} + (n-i)\gamma /2.
\end{align*}

The first inequality follows from weak submodularity since $(\text{OPT}_{i-1} \setminus i ) \subseteq Y_{i-1}\setminus ( X_{i-1} \cup i) $ which again has cardiality $n-i$. Since the same bound holds in either case we may now bound 

\begin{align*}
    \mathbb{E}[\nu ( \text{OPT}_{i-1}) - \nu ( \text{OPT}_{i})] &\leq  \frac{a_ib_i}{a_i + b_i} + (n-i)\gamma /2 \\
    &\leq \frac{1}{2}\frac{a_i^2 + b_i^2}{a_i + b_i}  + (n-i)\gamma /2 \\
    &= \frac{1}{2} \bigg \{ \mathbb{E}[ \nu(X_{i} ) - \nu(X_{i-1})+ \nu(Y_{i} )-\nu(Y_{i-1} ) )] + (n-i)\gamma /2 \bigg \}  + (n-i)\gamma /2 \\
    &= \frac{1}{2}\mathbb{E}[ \nu(X_{i} ) - \nu(X_{i-1})+ \nu(Y_{i} )-\nu(Y_{i-1} ) )]   + \frac{3}{4}(n-i)\gamma,
    \end{align*}
    
where the second inequality uses the fact that $2ab \leq a^2 + b^2$ for all $a,b \in \mathbb{R}$.

\end{proof}

\begin{proof}[Proof of Theorem \ref{thm: double greedy}]

Summing Lemma \ref{lemma: weak submodular building block lemma} over all $i$ yields,

\[ \sum_{i=1}^n \mathbb{E}[ \nu(\text{OPT}_{i-1} ) -  \nu(\text{OPT}_{i}) ] \leq \sum_{i=1}^n \bigg \{ \frac{1}{2}\mathbb{E}[ \nu(X_i) - \nu(X_{i-1}) + \nu(Y_i) - \nu(Y_{i-1}) ] + \frac{3}{4}(n-i)\gamma \bigg \}. \]

Both sides are telescoping sums. Recalling that $\text{OPT}_0 = \text{OPT}$ and $\text{OPT}_n = X_n = Y_n$ the sums collapse down to,

\begin{align*}
    \mathbb{E}[ \nu(\text{OPT})  -  \nu(X_n) ] &\leq  \mathbb{E}[ \nu(X_n) - \nu(X_{0}) + \nu(Y_n) - \nu(Y_{0}) ] +  \frac{3}{8}(n-1)n \gamma \\
    &\leq \frac{1}{2} \mathbb{E}[ \nu(X_n) + \nu(Y_n)] +  \frac{3}{8}(n-1)n \gamma \\
    &= \mathbb{E}[ \nu(X_n)] +  \frac{3}{8}(n-1)n \gamma,
\end{align*}

which rearranges to,

\[ \mathbb{E}[ \nu(X_n)] \geq  \frac{1}{2} \nu(\text{OPT})  - \frac{3}{16}(n-1)n \gamma .\]

\end{proof}

\subsection{Greedy guarantees for increasing weak submodular functions}

We call a set function $\nu $ \emph{increasing} if $\nu(S) \leq \nu(T)$ whenever $S \subseteq T $. We can greedily obtain an estimate for $ \max_{S : \abs{S} \leq k} \nu(S)$ by setting $S_0 = \varnothing$ and recursively computing $S_\ell $ by adding $\arg \max_{i \in [n]  }\nu(S_{\ell-1} \cup \{ i \})$ to $S_{\ell -1}$ for $\ell =1, \ldots , k$. For the remainder of this section we shall use $\text{OPT}$ to denode an element of the optimal solution set $\arg \max_{S : \abs{S} \leq k} \nu(S)$.

\begin{lemma}\label{lemma: weak submodular greedy result}
Let $\nu : 2^{[n]} \rightarrow \mathbb{R}_+$ be an increasing $\gamma$-weakly submodular function and $S_\ell \subseteq [n] $ be the set of size $\ell$ obtained by greedily optimizing $\nu$. Then for $k \leq n$ we have, 

\[ \nu (S_\ell) \geq (1 - e^{- \ell /k} ) \nu( \text{OPT}) - a\]

where $a = \frac{1}{2} k(k-1)  \bigg \{ 1 - \big ( 1 - \frac{1}{k}) ^\ell \bigg \} \gamma $. In particular, for $k = \ell$ we have,   \[\nu (S_k) \geq (1 - 1/e ) \nu( \text{OPT})  - a.\]
\end{lemma}

The bound is not vacuous so long as $  \nu( \text{OPT}) = \max_{S : \abs{S} \leq k} \nu(S)$ grows as $O(k^2)$.

\begin{proof}[Proof of Lemma \ref{lemma: weak submodular greedy result}]
Enumerate $\text{OPT} = \{ v_1, \ldots , v_k \}$. Then for each $i < \ell$, 

\begin{align}
    \nu(\text{OPT} ) & \leq \nu(\text{OPT} \cup S_i ) \label{eqn: submod mononton} \\
              &= \nu(S_i) + \sum_{j=1}^k  \bigg \{ \nu(S_i \cup  \{ v_1, \ldots , v_j \} ) - \nu(S_i \cup  \{ v_1, \ldots , v_{j-1}  \} ) \bigg \} \label{eqn: telescop sum}\\
              & =\nu(S_i) + \sum_{j=1}^k  \nu(v_j \mid S_i \cup \{ v_1, \ldots , v_{j-1} \})  \\
              & \leq \nu(S_i) + \sum_{j=1}^k \bigg \{ \nu(v_j \mid S_i  ) + (j-1) \log \gamma \bigg \} \label{eqn: bound greedy monotone} \\
              &\leq \nu(S_i) + \sum_{j=1}^k \big \{ \nu(S_{i+1}) - \nu(S_i) \big \} + \frac{1}{2} k (k-1) \log \gamma \label{eqn: use max} \\
              &\leq \nu(S_i) + k \big ( \nu(S_{i+1}) - \nu(S_i) \big ) + \frac{1}{2} k (k-1) \log \gamma
\end{align}

where $(\ref{eqn: submod mononton} )$ follows from monotonoicity of $\nu$, $(\ref{eqn: telescop sum})$ is a telescoping sum, $(\ref{eqn: bound greedy monotone})$ is obtained by bounding each term in the sum using Theorem \ref{thm: weak submodular property}, and $(\ref{eqn: use max} ) $ uses the fact that $S_{i+1}$ attains the maximal value of $\nu$ over all sets of size $i+1$. Defining $\delta_i = \nu(\text{OPT}) - \nu(S_i)$ we have therefore obtained $\delta_i \leq k( \delta_i - \delta_{i+1} ) + \frac{1}{2} k(k-1) \log  \gamma$. This rearranges to 

\[\delta_{i+1} \leq \bigg ( 1 - \frac{1}{k} \bigg ) \delta_i + \frac{1}{2} (k-1) \log  \gamma. \]

Unrolling this recursive relation we find that,

\begin{align*}
    \delta_\ell &\leq \bigg ( 1 - \frac{1}{k} \bigg )^\ell \delta_0 + \frac{1}{2} (k-1)  \sum_{i=0}^{\ell -1} \bigg ( 1 - \frac{1}{k} \bigg ) ^i \log  \gamma\\
    &\leq  e^{ - \ell / k} \delta_0 + \frac{1}{2} k(k-1)  \bigg \{ 1 - \big ( 1 - \frac{1}{k}) ^\ell \bigg \}\log  \gamma .\\
\end{align*}

Substituting back in the fact that $\delta_i = \nu(\text{OPT}) - \nu(S_i)$ and that $\nu( \varnothing ) \geq 0 $ and rearranging obtains the result.

\end{proof}

\begin{cor}\label{corollary: weak log-submodular greedy result}
Let $\nu : 2^{[n]} \rightarrow \mathbb{R}$ be an increasing $\gamma$-weakly log-submodular function and $S_\ell \subseteq [n] $ be the set of size $\ell$ obtained by greedily optimizing $\nu$. Then for $k \leq n$ we have, 

\[ \nu (S_\ell) \geq e^{-a}  \nu( \text{OPT}) ^{(1 - e^{- \ell /k} ) } \]

where $a = \frac{1}{2} k(k-1)  \bigg \{ 1 - \big ( 1 - \frac{1}{k}) ^\ell \bigg \}\log  \gamma $. In particular, for $k = \ell$ we have,   \[\nu (S_k) \geq e^{-a}   \nu( \text{OPT})^{(1 - 1/e ) }.\]
\end{cor}

\section{Further experimental results}\label{appendix: experiments}

In this section we report the empirical mixing time results for different spectra on the positive semi-definite matrix $L$ as described in Section \ref{section: experiments}.

\begin{figure}[H]
\centering
\textbf{(ii) $L$ has one big eigenvalue} \par\medskip
\begin{subfigure}{.5\textwidth}
  \centering
  \includegraphics[width=1.0\linewidth]{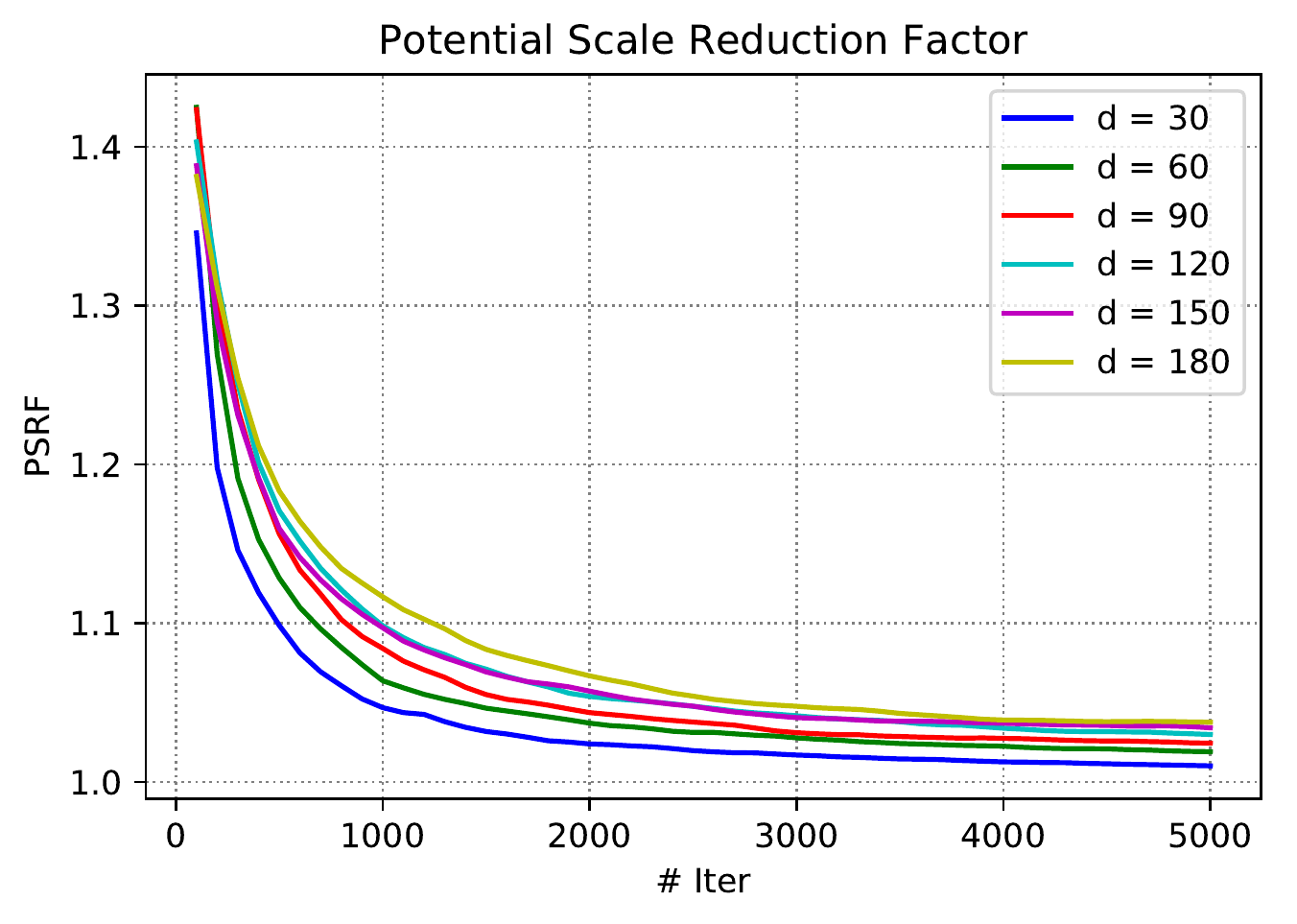}
  \caption{}
\end{subfigure}%
\begin{subfigure}{.5\textwidth}
  \centering
  \includegraphics[width=1.0\linewidth]{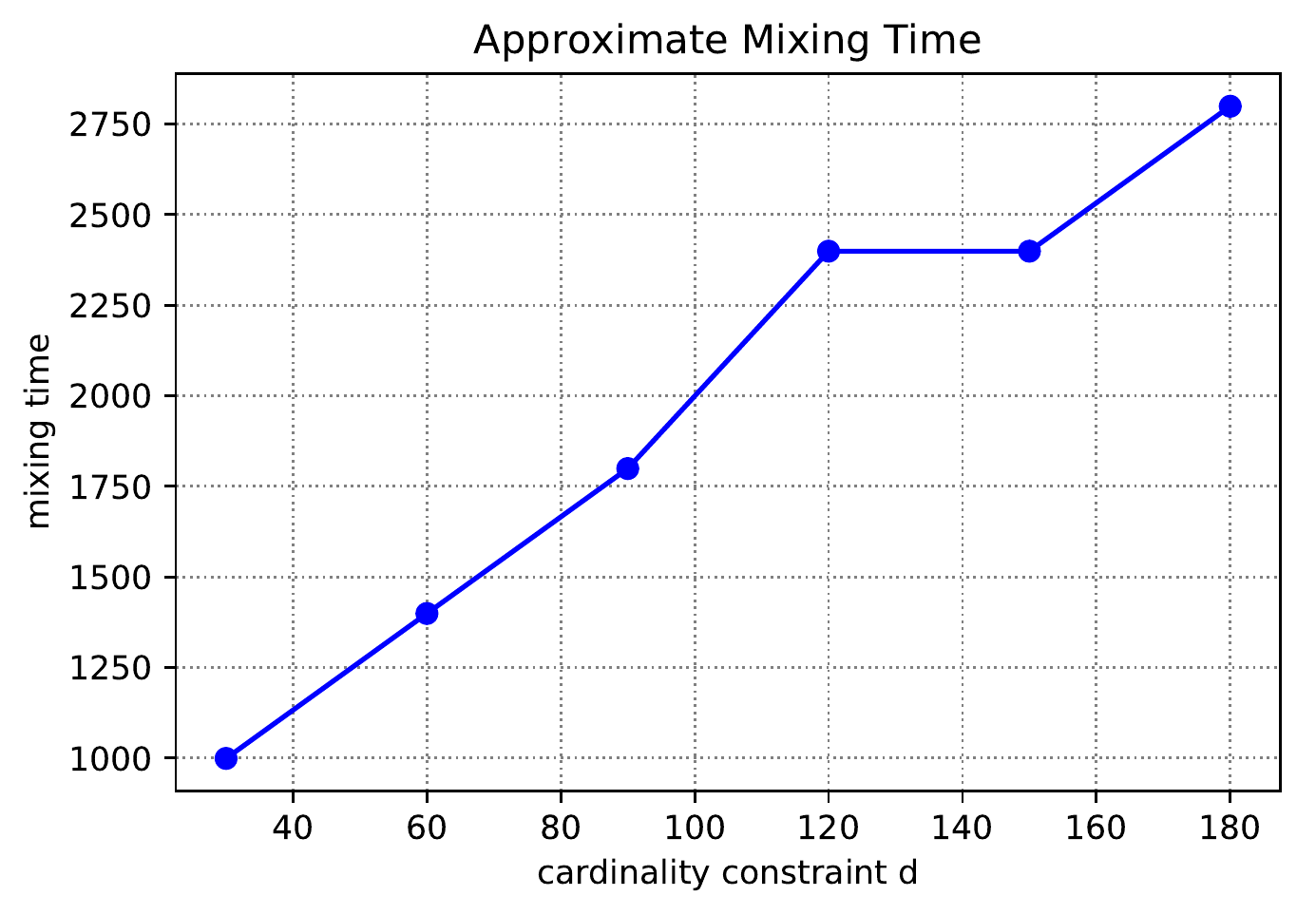}
  \caption{}
\end{subfigure}
\caption{Empirical mixing time analysis for sampling a ground set of size $n = 250$ and various cardinality constraints $d$, (a) the PSRF score for each set of chains, (b) the approximate mixing time obtained by thresholding at PSRF equal to $1.05$.}
\label{fig: PSRF and mix time varying d one big eval}
\end{figure}

\begin{figure}[H]
\centering
\textbf{(ii) $L$ has one big eigenvalue} \par\medskip
\begin{subfigure}{.5\textwidth}
  \centering
  \includegraphics[width=1.0\linewidth]{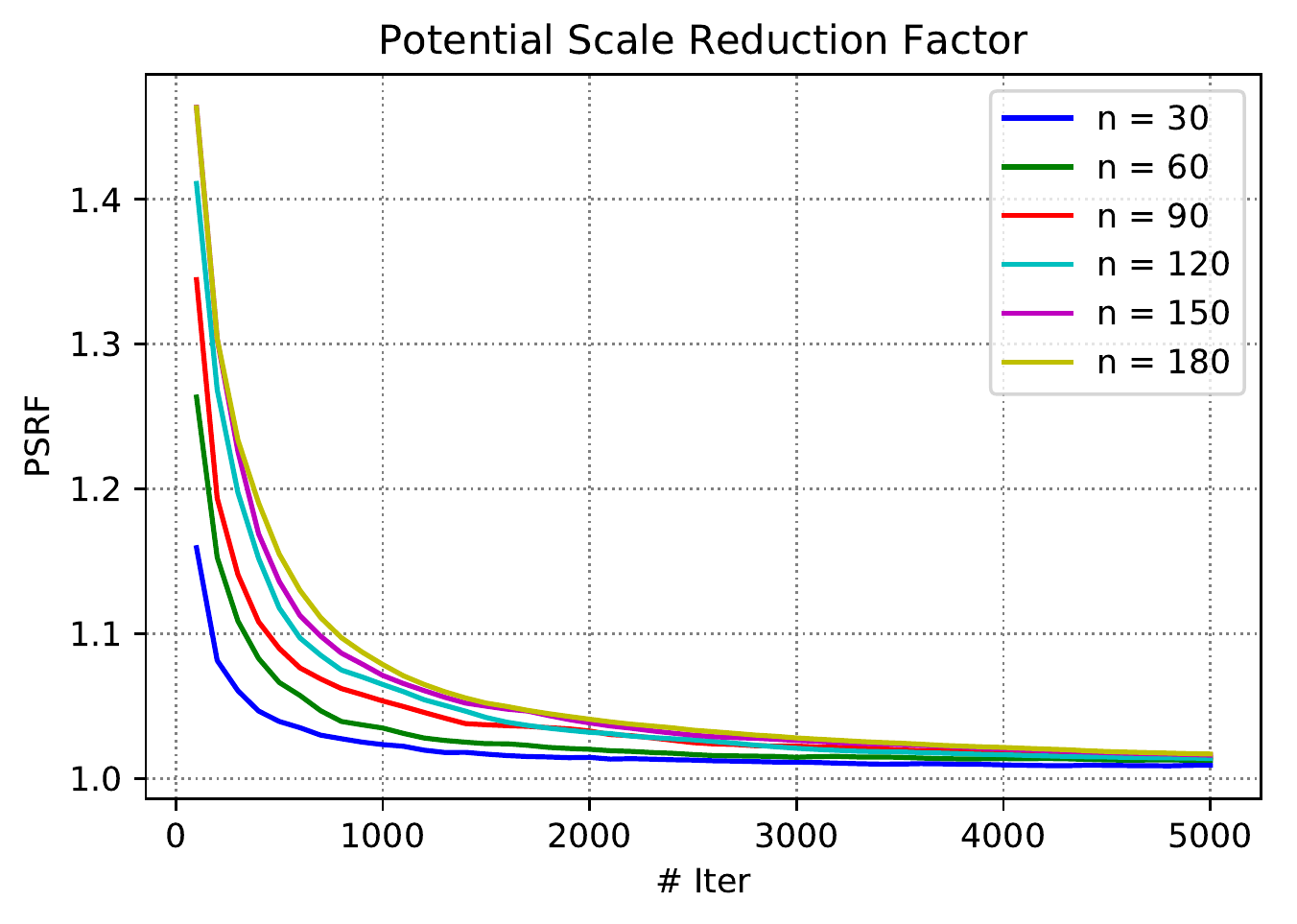}
  \caption{}
\end{subfigure}%
\begin{subfigure}{.5\textwidth}
  \centering
  \includegraphics[width=1.0\linewidth]{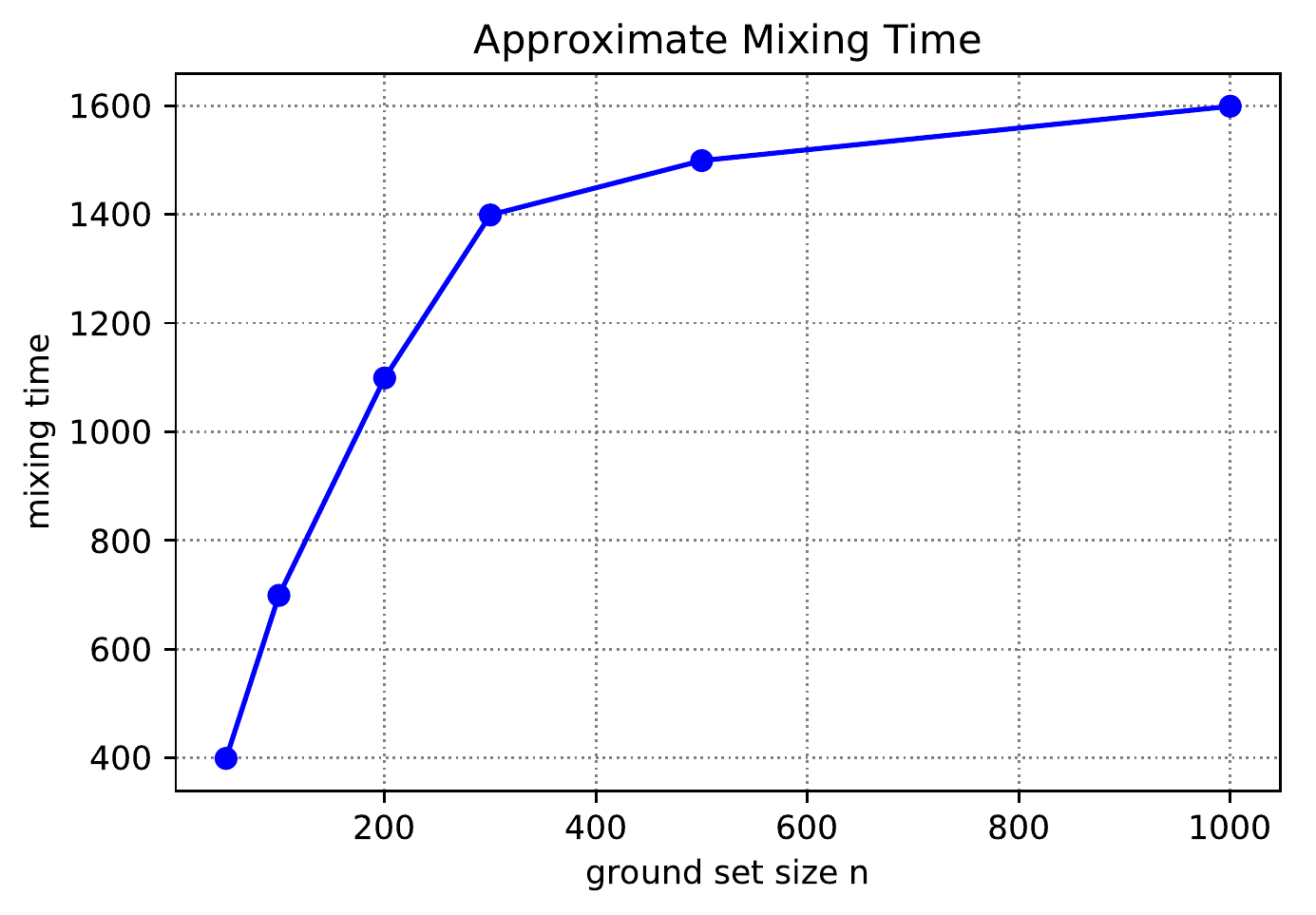}
  \caption{}
\end{subfigure}
\caption{Empirical mixing time analysis for sampling a set of size at most $d = 40$ for varying ground set sizes, (a) the PSRF score for each set of chains, (b) the approximate mixing time obtained by thresholding at PSRF equal to $1.05$.}
\label{fig: PSRF and mix time varying n one big eval}
\end{figure}

\begin{figure}[H]
\centering
\textbf{(iii) $L$ has a step in its spectrum} \par\medskip
\begin{subfigure}{.5\textwidth}
  \centering
  \includegraphics[width=1.0\linewidth]{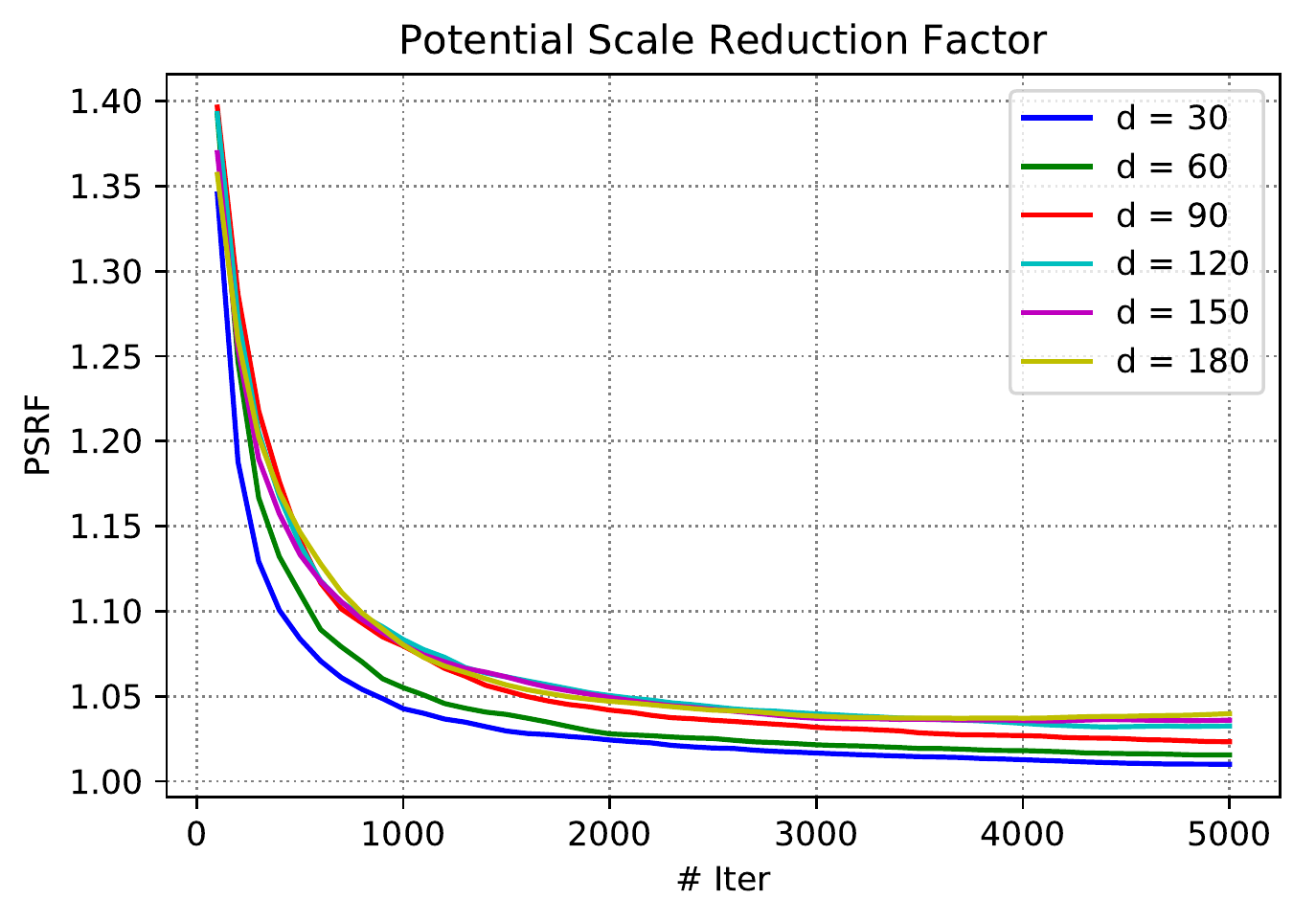}
  \caption{}
\end{subfigure}%
\begin{subfigure}{.5\textwidth}
  \centering
  \includegraphics[width=1.0\linewidth]{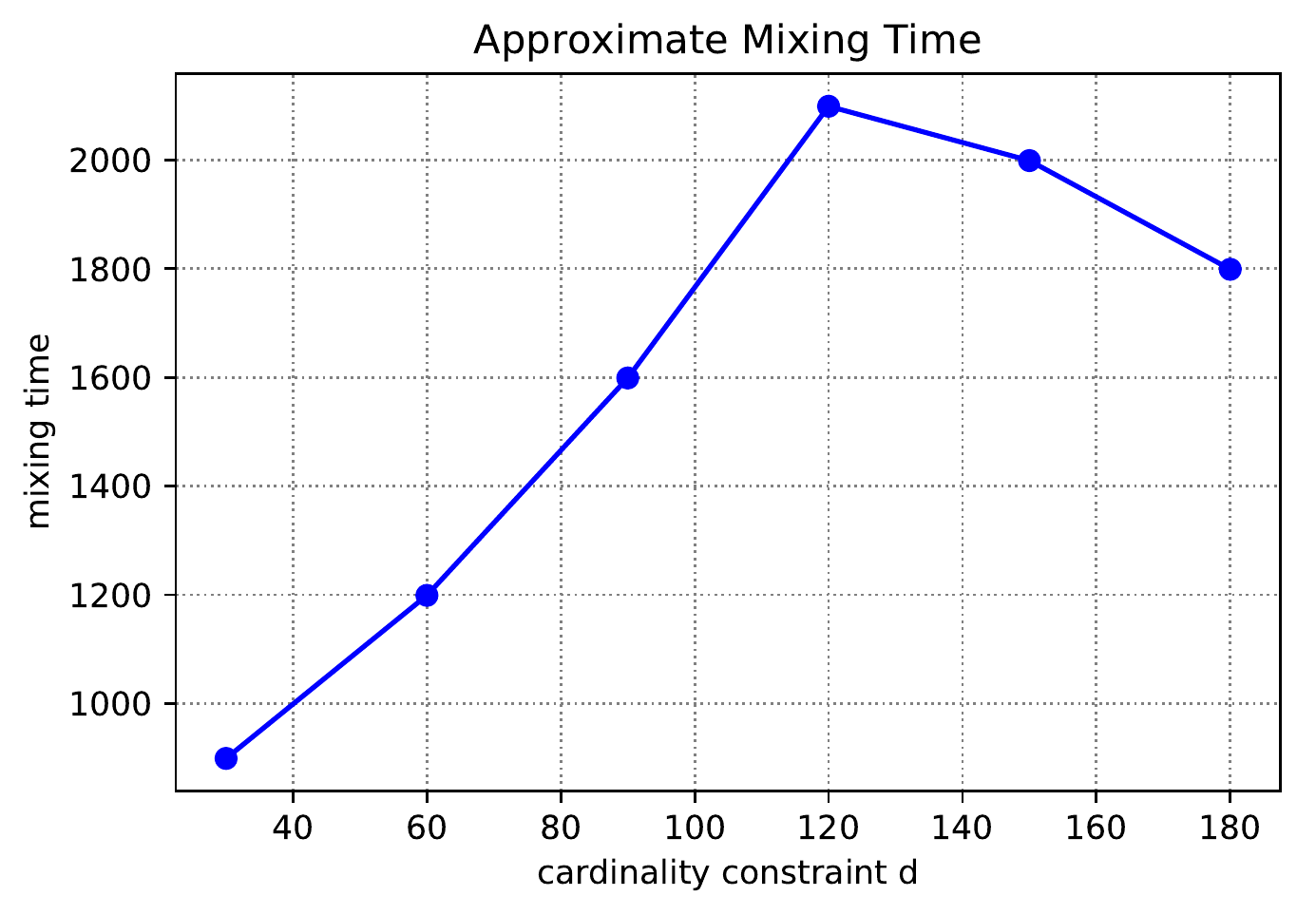}
  \caption{}
\end{subfigure}
\caption{Empirical mixing time analysis for sampling a ground set of size $n = 250$ and various cardinality constraints $d$, (a) the PSRF score for each set of chains, (b) the approximate mixing time obtained by thresholding at PSRF equal to $1.05$.}
\label{fig: PSRF and mix time varying d step spectrum}
\end{figure}

\begin{figure}[H]
\centering
\textbf{(iii) $L$ has a step in its spectrum} \par\medskip
\begin{subfigure}{.5\textwidth}
  \centering
  \includegraphics[width=1.0\linewidth]{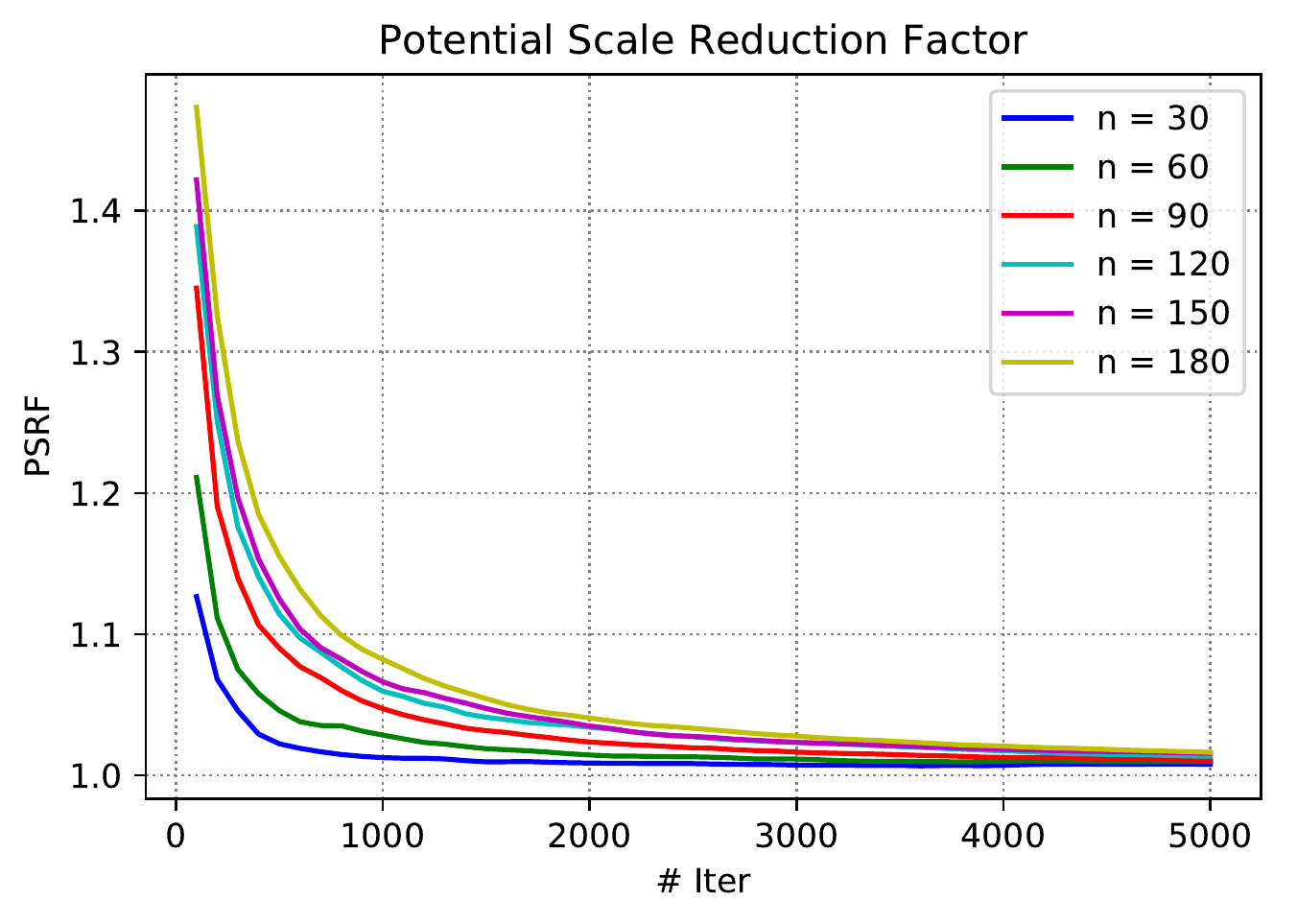}
  \caption{}
\end{subfigure}%
\begin{subfigure}{.5\textwidth}
  \centering
  \includegraphics[width=1.0\linewidth]{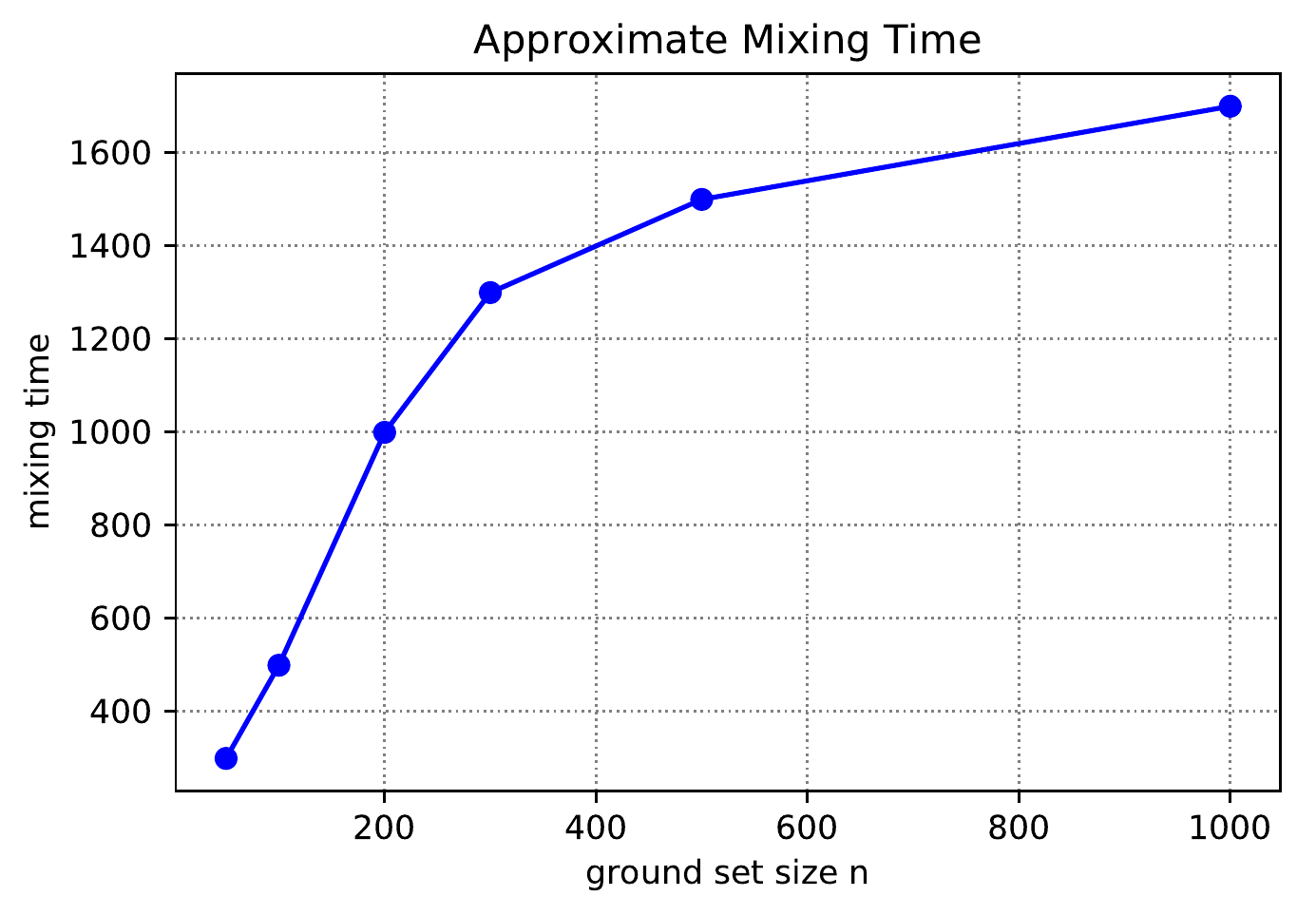}
  \caption{}
\end{subfigure}
\caption{Empirical mixing time analysis for sampling a set of size at most $d = 40$ for varying ground set sizes, (a) the PSRF score for each set of chains, (b) the approximate mixing time obtained by thresholding at PSRF equal to $1.05$.}
\label{fig: PSRF and mix time varying n step spectrum}
\end{figure}

\section{Supplementary material on strongly log-concave polynomials and log-Sobolev inequalities}\label{appendix: supplementary}

\subsection{Some theory for strongly log-concave polynomials}
In the interests of completeness this section recalls some recent results that we rely on for our analysis. We shall call a polynomial $f$ \emph{indecomposable} if it cannot be written as a sum $f = f_1 + f_2$ where $f_1, f_2$ are non-zero polynomials in disjoint sets of variables. In other words, the graph with nodes $\{ i \mid \partial_i f \neq 0 \}$ and edges $\{ (i,j) \mid \partial_i \partial_j f \neq 0 \} $ is connected. To state the next theorem we must introduce some more notation: for $\alpha \in \mathbb{N}^n$  define $\abs{\alpha} = \sum_{i=1}^n \alpha_i$.

 The following theorem provides checkable conditions for proving a polynomial is SLC.

\begin{theorem}\label{thm: anari log-concave characterization}\cite{anari2018log3}
Let $f \in \mathbb{R}_+[z_1, \ldots , z_n]$ be homogeneous and of degree $d \geq 2$. If the following two conditions hold, then $f$ is SLC, 

\begin{enumerate}
    \item For all $\alpha \in \mathbb{N}^n$ with $\abs{\alpha} \leq d-2$ the polynomial $\partial^\alpha f $ is indecomposable.
    \item For all $\alpha \in \mathbb{N}^n$ with $\abs{\alpha} = d-2$ the polynomial $\partial^\alpha f $ is log-concave on $\mathbb{R}_+^n$.
\end{enumerate}
\end{theorem}

\begin{theorem} \label{thm: Huh weak log-submodularity results}\cite{branden2019lorentzian}
If $f$ homogeneous and SLC then, 

\[ f(z) \partial_i \partial_j f(z) \leq 2 \bigg ( 1 - \frac{1}{d} \bigg) \partial_i f(z) \partial_j f(z) \]

for all $z \in \mathbb{R}_+^n$.

\end{theorem}

\begin{lemma}\cite{anari2018log3}\label{lemma: SLC closed under affine transformations}
Suppose $f \in \mathbb{R}[z_1, \ldots , z_n ]$ is homogenous and SLC and let $T : \mathbb{R}^n \rightarrow \mathbb{R}^n$ be such that if $z \in \mathbb{R}^n_+$ then so is $Tz$. Then $f \circ T \in \mathbb{R}[z_1, \ldots , z_n ]$  is also SLC.
\end{lemma}

\begin{lemma}\label{lemma: log conc iff one pos eval}
Suppose $f \in \mathbb{R}[z_1, \ldots , z_n ]$ is homogenous and $a \in \mathbb{R}_+$ such that $f(a) \neq 0$, and set $Q = \nabla^2f \mid_{z=a}$. Then the following are equivalent,

\begin{enumerate}
\item $ f$ is log-concave at $z=a$,
\item $(a^\top Q a) Q - (Qa)(Qa)^\top $ is negative semidefinite,
\item $Q$ has at most one positive eigenvalue.
\end{enumerate}
\end{lemma}

The equivalence of the first two statements is proven in \cite{anari2018log3}. The third is not proven by Anari et al. but is a simple consequence of a result from \cite{anari2018log3}. Finally we note the following characterization of the support of a homogenous SLC distribution. It is a special case of a result due to Br{\"a}nd{\'e}n and Huh (\cite{branden2019lorentzian}, Theorem $7.1$). Let $e_i$ denote the $i$th standard basis vector. We shall call a set $J \subseteq \mathbb{N}^n$ \emph{M-convex} if for any $\alpha, \beta \in J$ and any index $i$ such that $\alpha_i > \beta_i$, there exists an index $j$ such that $\alpha_j < \beta_j$ and $\alpha - e_i + e_j \in J$.

\begin{theorem} \cite{branden2019lorentzian}\label{thm: support is set of bases of a matroid}
Suppose $f = \sum_{ \abs{\alpha} = d} c_\alpha z^\alpha$ is SLC. Then the support $\{ \alpha \in \mathbb{N}^d \mid c_\alpha \neq 0 \}$ of $f$ is an M-convex set. 
\end{theorem}

\subsection{Log-Sobolev inequalities for bounding mixing times}

Throughout this section we consider a finite. state space $\mathcal{X}$ and all distributions, kernels, and other functions mentioned will be defined on it. We define the Dirichlet form of a function $f$ with respect to a Markov chain to be

\[ \mathcal{E}(f,f) = \frac{1}{2} \sum\limits_{x,y \in \mathcal{X} } \abs{ f(x) - f(y) }^2 Q(x,y) \pi(x). \] 

We also introduce the following entropy-like quantity

\[ \mathcal{L}(f) = \sum\limits_{x \in \mathcal{X} } f(x) ^2 \log \bigg ( \frac{ f(x)^2 }{ \norm{f}_2^2 } \bigg ) \pi(x). \]

\begin{defn}

The log-Sobolev constant of the Markov chain $(Q,\pi)$ is the largest $\alpha > 0 $ such that

\[ \alpha  \mathcal{L}(f)  \leq \mathcal{E}(f,f) \] 

for all $f$.

\end{defn}

\begin{theorem}\cite{diaconis1996logarithmic}\label{thm: Diaconis comparting log-sobolev}
Let $(Q, \pi)$ and  $(Q', \pi')$ with Dirichlet forms and log-Sobolev constants  $ \mathcal{E}, \alpha$ and $ \mathcal{E}', \alpha'$ respectively. Suppose there exists constants $a,A > 0$ such that 

\begin{equation*}
\mathcal{E}' \leq A \mathcal{E}
   \quad\text{and}\quad 
\pi \leq a \pi '.
\end{equation*}

Then 
\begin{equation*}
\alpha' \leq aA \alpha.
\end{equation*}
\end{theorem}

So, suppose you had a chain of interest  $(Q, \pi)$ and happened to already know what $\alpha'$ was for some other chain $(Q', \pi')$, then by determining the constants $a,A > 0$ one obtains a lower bound on $\alpha$. This immediately yields an upper bound on the mixing time of  $(Q, \pi)$.

\end{document}